\documentclass{article}

\title{The late-stage training dynamics of (stochastic) subgradient descent on homogeneous neural networks}
\usepackage{times}

\usepackage{authblk}

\author[1]{Sholom Schechtman}
\author[2]{Nicolas Schreuder}
\affil[1]{SAMOVAR, Télécom SudParis, Institut Polytechnique de Paris}
\affil[2]{CNRS, LIGM}

\title{The late-stage training dynamics of (stochastic) subgradient descent on homogeneous neural networks}

\usepackage{times}

\usepackage{a4wide}

\usepackage{amsmath, amsthm, amssymb}

\usepackage{shortcuts}

\usepackage{mathtools}
\usepackage{natbib}
\setcitestyle{authoryear}

\usepackage{url}
\usepackage[colorlinks=true, linkcolor=blue, citecolor=blue]{hyperref}

\DeclareMathOperator{\Jac}{Jac}

\newcommand{\tgamma}{\tilde{\gamma}}

\newcommand{\bg}{\bar{g}}
\newcommand{\bgamma}{\bar{\gamma}}
\DeclareMathOperator{\ReLU}{ReLU}
\DeclareMathOperator{\LeakyReLU}{LeakyReLU}
\newtheorem{assumption}{Assumption}
\usepackage{color}
\usepackage{csquotes}
\usepackage{comment}
\usepackage{enumitem}

\newtheorem{theorem}{Theorem}
\newtheorem{definition}[theorem]{Definition}
\newtheorem{proposition}[theorem]{Proposition}
\newtheorem{lemma}[theorem]{Lemma}
\newtheorem{remark}[theorem]{Remark}

\newcommand{\edi}[1]{{ #1}}

\begin{document}

\maketitle
\begin{abstract}%
  We analyze the implicit bias of constant step stochastic subgradient descent (SGD). We consider the setting of binary classification with homogeneous neural networks -- a large class of deep neural networks with $\ReLU$-type activation functions such as MLPs and CNNs without biases. We interpret the dynamics of normalized SGD iterates as an Euler-like discretization of a conservative field flow that is naturally associated to the normalized classification margin. Owing to this interpretation, we show that normalized SGD iterates converge to the set of critical points of the normalized margin at late-stage training (i.e., assuming that the data is correctly classified with positive normalized margin). 
  Up to our knowledge, this is the first extension of the analysis of the \cite{Lyu_Li_maxmargin} on the discrete dynamics of gradient descent to the nonsmooth and stochastic setting. Our main result applies to binary classification with exponential or logistic losses. We additionally discuss extensions to more general settings. 
\end{abstract}

\section{Introduction}
Modern deep learning architectures are typically \emph{overparameterized} relative to the number of training samples, leading to infinitely many weight configurations that interpolate the training data. Traditional learning theory suggests such solutions should generalize poorly \citep{vapnik, belkin2019reconciling}. Yet, empirical evidence shows that interpolating solutions found through training often generalize well \citep{zhang2021understanding}.
This observation has led to the \emph{implicit bias} hypothesis, which posits that gradient-based optimization algorithms act as implicit regularizers, guiding models toward parameters that generalize well \citep{NeyshaburTS14,GunasekarLSS18}. Understanding this bias is key to explaining the success of deep learning and remains an active research area.

While a general theoretical characterization for deep neural networks remains out of reach, progress has been made in simplified settings. One such setting is the class of positively homogeneous neural networks (NNs), where a network $\Phi(\cdot; w)$ satisfies $\Phi(x; \lambda w) = \lambda^L \Phi(x; w)$ for any $\lambda > 0$, weights $w$, and input features $x$. This class encompasses feedforward NNs of depth $L$ without biases and with piecewise linear activation functions, such as MLPs and CNNs.


In the spirit of early works by \cite{Lyu_Li_maxmargin,ji2020directional}, this paper focuses on binary classification problem with such networks. Given a labeled dataset $((x_i,y_i))_{i \leq n}$ in $\mathbb{R}^p \times \{-1, 1\}$ and \edi{a non-negative loss function $l$ strictly decreasing to zero} (e.g., logistic or exponential loss), the objective is to minimize the empirical risk,
\begin{equation}\label{def:emp_risk_intro}
  \cL(w) \coloneqq \frac{1}{n}\sum_{i=1}^N l(y_i \Phi(x_i; w))\, .
\end{equation}
The quality of a prediction fora  given weight vector $w$ is measured by the margin,
\begin{equation}\label{def:margin_intro}
  \sm(w) \coloneqq \min_{1 \leq i \leq n} y_i \Phi(x_i;w)\, .
\end{equation}
Notably, the training set is correctly classified if $\sm(w){>}0$.
Two key observations follow. First, if some weights $w$ sastify $\sm(w){>}0$, then for any $\lambda{>}1$, $\cL(\lambda w){<}\cL(w)$. Consequently, $\inf_{w}\cL(w) = 0$, and reaching this infimum requires the norm of the weights  to grow to infinity. Second, due to homogeneity, prediction quality depends only on the direction $\bar{w} := w / \norm{w}$. This naturally raises the question of how $\bar{w}$ evolves during training and what properties characterize its limit points.

A step toward answering this question was taken by \cite{Lyu_Li_maxmargin}, who analyzed the dynamics of subgradient flow (GF) and, under additional smoothness assumptions, gradient descent (GD) on the empirical risk. Assuming that at some point during training $\sm(w) > 0$ (i.e., the training error is zero), they showed that the parameter norm grows to infinity and that the direction $\bar{w}$ converges to the set of \emph{KKT points} of a constrained optimization problem related to margin maximization. Later, \cite{ji2020directional} proved  $\bar{w}$ actually converges to a \emph{unique} KKT point considering the GF dynamics. These results highlights a clear implicit bias: even after achieving zero training error, the dynamics continue evolving until reaching a form of optimality in the normalized margin.

While the works of \cite{Lyu_Li_maxmargin,ji2020directional} provide key insights into the training dynamics of homogeneous networks, they leave open an important question: does the same implicit bias phenomenon extend to (stochastic) subgradient descent (SGD)? GF describes the limiting continuous-time dynamics of SGD, but the discrete-time analysis in \cite{Lyu_Li_maxmargin} does not account for stochasticity or nonsmoothness. As a result, their framework does not extend to neural networks with $\ReLU$ or $\LeakyReLU$ activations, which are widely used in practice \citep{goodfellow2016deep,fleuret2023little}.

\paragraph{Main contributions} We address this open question by identifying a setting in which the limit points of the normalized directions $u_k := w_k / \norm{w_k}$ can be characterized, where $(w_k)$ is generated by SGD with a \emph{constant step-size} on the empirical risk~\eqref{def:emp_risk_intro} for $L$-homogeneous \emph{nonsmooth} networks. Our analysis covers both the exponential and logistic losses. Since we are interested in the limit points of the normalized directions, we focus the late-stage training dynamics, \textit{i.e.} after the training data is correctly classified. We \edi{define}
this stage as the regime in which the normalized margin remains positive, 
\edi{as formalized later by the event $\cE$ in Equation~\eqref{eq:def_event} and discussed in Section~\ref{sec:sett}}. Our main contributions can be summarized as follows:

\begin{itemize}
\item[1.] We show that the dynamics of $(u_k)$ can be interpreted as a Euler-like discretization (or stochastic approximation) of a 
gradient flow inclusion, $\dot{\su}(t) \in  \bar{D}_s(\su(t))$, where $\bar{D}_s$ is a conservative set-valued field of the margin \emph{restricted} to the unit sphere $\sm_{|\bbS^{d-1}}$. This gradient-like object, recently introduced by \cite{bolte2021conservative}, naturally appears in the analysis.
 In particular, we show in Proposition~\ref{prop:stoch_approx_exp_log} that, \edi{approximately},
\begin{equation}\label{eq:sto_appro_intro}
  u_{k+1} \in  u_k + \bgamma_k \bar{D}_s(u_k) + \bgamma_k e_k\, ,
\end{equation}
where $(e_k)$ is a sequence of (stochastic) perturbations, that diminishes as $k \rightarrow + \infty$.
\item[2.] As a consequence, leveraging recent results on stochastic approximations of differential inclusions (\cite{benaim2006dynamics,dav-dru-kak-lee-19,bolte2021conservative}), we establish in Theorem~\ref{thm:main} that $(u_k)$ converges to the set of $\bar{D}_s$-critical points $\cZ_s := \{u \in \bbS^{d-1}: \bar{D}_s(u) = 0 \}$. In our setting, $\cZ_s$ coincides \emph{exactly} with the KKT points of \cite{Lyu_Li_maxmargin}.
\end{itemize}
From a mathematical perspective, our techniques differ from those of \cite{Lyu_Li_maxmargin,ji2020directional}. 
Rather than constructing a smooth approximation of the margin, we directly consider the margin and interpret the dynamics of the normalized SGD iterates as a stochastic approximation. This, in contrast to these previous works, allows us to incorporate both stochasticity and nonsmoothness in our analysis. Finally, although we focus on exponential-type losses, we also present a more general setting in Appendix~\ref{app:gen_sett} where our proof techniques apply.

\paragraph{Organization of the paper} We begin by reviewing related works in Section~\ref{sec:rw}. Then, in Section~\ref{sec:preli}, we introduce key technical concepts for nonsmooth analysis. Section~\ref{sec:sett} presents the main setting we consider. Our main result and its proof are detailed in Sections~\ref{sec:main} and~\ref{pf:sto_app_explog}, respectively. Finally, additional proofs and supplementary materials are provided in the appendices.

\section{Related works}\label{sec:rw}
We review related works  on the implicit bias of neural networks trained with gradient-based algorithms and their connection to margin maximization. The discussion follows an increasing order of flexibility, starting with logistic regression--viewed as a one-layer network-- then progressing to linear and homogeneous networks and finally addressing non-homogeneous architectures. For a comprehensive survey on the topic, we refer the reader to \cite{vardi2023implicit}.

\paragraph{Logistic regression} Training a logistic regression model can be viewed as training a one-layer neural network or fine-tuning the final linear layer of a deep network, making it a natural starting point for studying the training dynamics of deep networks. Motivated by this perspective, a first line of research has explored the properties of (stochastic) gradient descent iterates for linear logistic regression.
\cite{soudry2018implicit} showed that when learning parameters  $w$  using gradient descent on the logistic loss--or more generally, on exponentially-tailed losses--over separable data, the direction of $w$  converges to the $\ell_2$-max margin solution, while its norm grows to infinity. \cite{nacson2019stochastic} extended this result to stochastic gradient descent with a fixed learning rate, demonstrating similar behavior. \cite{ji2018risk} further considered non-linearly separable data, showing that the direction of $w$ converges to the maximum margin predictor of the largest linearly separable subset.

\paragraph{Linear networks} Linear networks, composed of stacked linear layers $W_D \times \cdots \times W_1$ without activation functions, represent linear functions, but their parameterization strongly influences learning dynamics and implicit bias.  
\cite{gunasekar2018implicit} analyzed gradient descent on fully connected and convolutional linear networks, showing convergence in direction to the maximum-margin solution. \cite{ji2018gradient} extended this result by proving predictor convergence under weaker assumptions and establishing asymptotic singular vector alignment. \cite{nacson2019convergence} examined conditions on the loss function ensuring that gradient descent iterates converge to the $\ell_2$-maximum margin solution. \cite{yun2020unifying} later developed a unified framework for gradient flow in linear tensor networks, covering fully connected, diagonal, and convolutional architectures. 

\paragraph{Homogeneous networks} Linear networks cannot model architectures with nonlinear activations, limiting their applicability. Homogeneous networks, including fully connected and convolutional architectures without bias terms and with ReLU activations, have been considered to address this limitation.  
In this setting, \cite{du2018algorithmic} proves that gradient flow enforces the differences between squared norms across different layers to remain invariant without any explicit regularization.
Assuming training error converges to zero and parameters converge in direction, \cite{nacson2019lexicographic} established that rescaled parameters reach a first-order KKT point of a maximum-margin problem. 
\cite{Lyu_Li_maxmargin} showed that for networks trained with exponential or logistic loss, gradient flow converges to a KKT point of the maximum-margin problem in parameter space. \cite{vardi2022margin} examined the (local) optimality of these KKT points, identifying conditions for local or global optimality. 
\edi{\cite{chizat2020implicit}  improve these results for infinitely wide two-layer neural networks with homogeneous activations,  characterising the learnt classifier as the solution of a convex max-margin problem.}
\cite{ji2020directional} proved that the direction of gradient flow converges to a unique point. Finally, \cite{wang2021implicit} analyzed margin maximization in homogeneous networks trained with adaptive methods such as Adam.

\paragraph{Beyond homogeneous networks}
\cite{kunin2022asymmetric, le2022training} extend maximum-margin bias and alignment results from linear and homogeneous to quasi-homogeneous networks, allowing to consider networks with biases and residual connections.  \edi{\cite{cai2024large} analyzes the dynamics of large-step-size gradient descent for training predictors that satisfy a near-homogeneity condition. They identify two distinct training phases: an initial phase in which the empirical risk oscillates, followed by a second phase where it decreases monotonically, and examines the role of the step size in driving the phase transition.}

\section{Preliminaries}\label{sec:preli}
\label{sec:cons_field}


Although feedforward and convolutional NNs with ReLU or LeakyReLU activations are inherently nonsmooth—making their analysis more challenging—they exhibit some form of regularity. In particular, they belong to the important class of semialgebraic functions, which we now introduce.


\paragraph{Semialgebraic functions.}
A set $A \subset \bbR^d$ is semialgebraic if it can be expressed as a finite union of intersections of sets of the form $\{ Q(x) \leq 0\}$, where $Q : \bbR^d \rightarrow \bbR$ is a polynomial. A (set-valued) function is semialgebraic if its graph is a semialgebraic set. Examples include piecewise polynomials, rational functions and functions such as $x \mapsto x^{q}$, where $q$ is any rational number.

Semialgebraic functions are closed under composition and remain stable under a finite number of operations, including  $\{+, -, \times, \max, \min, \circ, \circ^{-1}\}$. From these properties it is clear that feedforward and convolutional NNs with piecewise linear activation functions such as $\ReLU$ are semialgebraic functions.  Notably, the class of semialgebraic functions is a particular case of functions definable in an o-minimal structure. We discuss o-minimality in Appendix~\ref{app:omin} and emphasize that every statement of the paper which involves semialgebraicity might be replaced by definability.

Although semialgebraic functions exhibit some regularity, they are not necessarily smooth. To analyze their behavior in optimization settings, we rely on a more general notion of differentiation--- \emph{conservative set-valued fields}, which we now define.

\paragraph{Conservative set-valued fields.}
Conservative set-valued fields were introduced in \cite{bolte2021conservative} as an elegant framework for describing the output of automatic differentiation provided by numerical libraries such as TensorFlow and PyTorch (\cite{tensorflow2015-whitepaper,paszke2017automatic}). 
They serve as a fundamental tool for analyzing first-order methods in non-smooth settings (see, e.g., \cite{bolte2023subgradient}).

\begin{definition}[\cite{bolte2021conservative}]\label{def:cons_f}
  A graph-closed, locally bounded set-valued map\footnote{A map $D: \bbR^d \rightrightarrows \bbR^d$ is set-valued if, for every $w \in \bbR^d$, $D(w)$ is a set in $\bbR^d$. See Appendix~\ref{app:interp} for various definitions related to such maps.} $D: \bbR^d \rightrightarrows \bbR^d$ with nonempty values is a \emph{conservative field} for the \emph{potential} $\cL: \bbR^d \rightarrow \bbR$, if for any absolutely continuous curve $\sw: [0, 1] \rightarrow \bbR^d$, it holds for almost every $t \in [0,1]$ that 
  \begin{equation*}
    \frac{\dif}{\dif t} \cL(\sw(t)) = \scalarp{v}{\dot{\sw}(t)} \quad \textrm{ for all $v \in D(\sw(t))$} \, .
  \end{equation*} 
  Functions that are potentials of some conservative field are called \emph{path differentiable}.
\end{definition}
Given a continuously differentiable potential $\cL$, the set-valued map $w \rightrightarrows \{ \nabla \cL(w)\}$ is \edi{an} obvious example of a conservative field. 
For semialgebraic functions, two important examples of conservative set-valued fields are the Clarke subgradient and the output of backpropagation.

\paragraph{Clarke subgradients.}
Semialgebraic functions always admit a conservative field, the Clarke subgradient. This result, initially proven in \cite{drusvyatskiy2015curves}, builds on the work of \cite{bolte2007clarke}.

\begin{definition}[Clarke subgradient {\cite{cla-led-ste-wol-livre98}}]
  Let $\cL: \bbR^d \rightarrow \bbR$ be a locally Lipschitz function. The Clarke subgradient of $\cL$ at $w \in\bbR^d$ is defined as
  \begin{equation}\label{eq:def_clarke}
    \partial \cL(w) := \conv \{ v \in \bbR^d: \textrm{ there exist $w_k \rightarrow w$, with $\cL$ differentiable at $w_k$ and $\nabla \cL(w_k) \rightarrow v$}\}\, ,
  \end{equation}
  where we denote by $\conv(A)$ the convex closure of a set $A$.
\end{definition}
For semialgebraic function $\cL$, $\partial \cL$ is the minimal convex-valued semialgebraic conservative field (\cite{drusvyatskiy2015curves, bolte2021conservative}). That is, for any conservative set-valued field $D$ with semialgebraic potential $\cL$, it holds for all $w$ that $\partial \cL(w) \subset \conv(D(w))$\footnote{Note that if $D$ is a conservative field, then the mapping $w \rightrightarrows \conv D(w)$ is also conservative.}. 

\paragraph{Backpropagation.} When applied to nonsmooth functions, the backpropagation algorithm formally applies the chain rule, replacing the derivative with an element of the Clarke subgradient when necessary. Although the Clarke subgradient of a composition $f \circ g$ is not necessarily equal to $\partial f \times \partial g$, the product of conservative mappings remains conservative (see Proposition~\ref{pr:comp_cons} in Appendix~\ref{app:conserv}).
As a result, \cite[Section 5]{bolte2021conservative} show that if a (possibly nonsmooth) semialgebraic function is defined through a computational graph (such as in a neural network), backpropagation produces an element of a conservative field. Consequently, the training a neural network via (S)GD is actually a (stochastic) conservative field descent.

\section{Problem setting}\label{sec:sett}
We study (stochastic) gradient descent on the empirical risk
\begin{equation*}
\cL(w) = \frac{1}{n}\sum_{i=1}^n l(p_i(w))\, ,
\end{equation*}
where the loss function $l$ and the functions  $(p_i)_{i=1}^n$  are specified in the following assumptions. Note that the empirical risk for binary classification from Equation~\eqref{def:emp_risk_intro} is a special case of the above objective.

\begin{assumption}\label{hyp:loss_exp_log}\phantom{=}
  \begin{enumerate}[label=\roman*)]
    \item The loss is either the exponential loss, $l(q) = e^{-q}$, or the logistic loss, $l(q) = \log(1{+}e^{-q})$.
    \item There exists an integer $L \in \mathbb{N}^*$  such that, for all $1 \leq i \leq n$, the function $p_i$ is \edi{positively} $L$-homogeneous\footnote{We recall that a mapping $f : \mathbb{R}^d \rightarrow \mathbb{R}$ is positively $L$-homogeneous if $f(\lambda w) = \lambda^L f(w)$ for all $w \in \mathbb{R}^d$ and $\lambda >0$.}, locally Lipschitz continuous and semialgebraic.
  \end{enumerate}
\end{assumption}
If the $p_i$'s were differentiable with respect to $w$, the chain rule would guarantee that
\begin{align*}
\nabla \mathcal{L}(w) = \frac{1}{n}\sum_{i=1}^n  l'(p_i(w)) \nabla p_i(w)\enspace.
\end{align*}
However, we only assume that the $p_i$'s are semialgebraic. While we could consider Clarke subgradients, the Clarke subgradient of operations on functions (e.g., addition, composition, and minimum) is only contained within the composition of the respective Clarke subgradients. This, as noted in Section~\ref{sec:cons_field}, implies that the output of backpropagation is usually not an element of a Clarke subgradient but a selection of some conservative set-valued field.
Consequently, for $1\leq i \leq n$, we consider $D_i : \bbR^d \rightrightarrows\bbR^d$, a conservative set-valued field of $p_i$, and a function $\sa_i : \bbR^d \rightarrow \bbR^d$ such that for all $w \in \bbR^d$, $\sa_i(w) \in D_i(w)$. Given a step-size $\gamma >0$, gradient descent (GD)\footnote{More precisely, this refers to conservative gradient descent. We use the term GD for simplicity, as conservative gradients behave similarly to standard gradients.} is then expressed as
\begin{equation*}\label{eq:gd_new}\tag{GD}
  w_{k+1} = w_k - \frac{\gamma}{n} \sum_{i=1}^n l'(p_i(w_k))\sa_i(w_k)\,.
\end{equation*}
For its stochastic counterpart, stochastic gradient descent (SGD), we fix a batch-size $1\leq n_b \leq n$. At each iteration $k \in \bbN$, we randomly and uniformly draw a batch $B_k \subset \{1, \ldots, n \}$ of size $n_b$. The update rule is then given by 
\begin{equation*}\label{eq:sgd_new}\tag{SGD}
  w_{k+1} = w_k -  \frac{\gamma}{n_b}\sum_{i\in B_k} l'(p_i(w_k)) \sa_i(w_k)\, .
\end{equation*}
The considered conservative set-valued fields will satisfy an Euler lemma-type assumption.
\begin{assumption}\phantom{=}\label{hyp:conserv}
  For every $i \leq n$, $\sa_i$ is measurable and $D_i$ is semialgebraic. Moreover, for every $w \in \bbR^d$ and $\lambda \geq 0$, $\sa_i(w)  \in D_i(w)$,
  \begin{equation*}
    D_i(\lambda w) = \lambda^{L-1} D_i(w)\, , \textrm{ and } \quad   L p_i(w) = \scalarp{\sa_i(w)}{w}\, .
  \end{equation*}
\end{assumption}
Having in mind the binary classification setting, in which $p_i(w) = y_i \Phi(x_i, w)$, we define the margin
\begin{equation}\label{def:marg}
  \sm: \bbR^d \rightarrow \bbR, \quad \sm(w) = \min_{1\leq i \leq n} p_i(w)\, .
\end{equation}
It quantifies the quality of a prediction rule $\Phi(\cdot, w)$. In particular,  the training data is perfectly separated when $\sm(w) >0$. A binary prediction for $x$ is given by the sign of $\Phi(x, w)$, and under the homogeneity assumption, it depends only on the normalized direction $w / \norm{w}$. Consequently, we will focus on the sequence of directions $u_k := w_k / \norm{w_k}$. 

\edi{Finally, all of our results hold on the event 
\begin{align}\label{eq:def_event}
  \cE:= [\liminf \sm(u_k) >0],
\end{align}
which ensures that the normalized directions $(u_k)$ have stabilized in a region where the training data is correctly classified.
}
Before presenting our main result, we comment on our assumptions.

\paragraph{On Assumption~\ref{hyp:loss_exp_log}.} As discussed in the introduction, the primary example we consider is when $p_i(w) = y_i \Phi(x_i;w)$ is the signed prediction of a feedforward neural network without biases and with piecewise linear activation functions on a labeled dataset $((x_i,y_i))_{i \leq n}$. In this case,
\begin{equation}\label{eq:NN}
 p_i(w) = y_i \Phi(w;x_i) = y_i V_L(W_L) \sigma(V_{L-1}(W_{L-1}) \sigma(V_{L-1}(W_{L-2}) \ldots \sigma(V_{1}(W_1 x_i))))\, ,
\end{equation}
where $w = [W_1, \ldots, W_L]$, $W_i$ represents the weights of the $i$-th layer, $V_i$ is a linear function in the space of matrices (with $V_i$ being the identity for fully-connected layers) and $\sigma$ is a coordinate-wise activation function such as $z \mapsto \max(0,z)$ ($\ReLU$), $z \mapsto \max(az, z)$ for a small parameter $a>0$ (LeakyReLu) or $z \mapsto z$. Note that the mapping $w \mapsto p_i(w)$ is semialgebraic and $L$-homogeneous for any of these activation functions. Regarding the loss functions, the logistic and exponential losses are among the most commonly studied and widely used. In Appendix~\ref{app:gen_sett}, we extend our results to a broader class of losses, including $l(q) = e^{-q^a}$ and $l(q) = \ln (1 + e^{-q^a})$ for any $a \geq 1$.

\paragraph{On Assumption~\ref{hyp:conserv}.} Assumption~\ref{hyp:conserv} holds automatically  if $D_i$ is the Clarke subgradient of $p_i$. Indeed, at any vector $w \in \bbR^d$, where $p_i$ is differentiable it holds that $p_i(\lambda w) = \lambda^{L} p_i(w)$. \edi{Differentiating  with respect to to $w$ and $\lambda$, respectively, and noting that $p_i$ remains differentiable at $\lambda w$ due to homogeneity,} we obtain $\lambda \nabla p_i(\lambda w) = \lambda^{L} \nabla p_i(w)$ and $\scalarp{\nabla p_i(\lambda w)}{w} = L \lambda^{L-1} p_i(w)$. \edi{The second equality follows by taking $\lambda =1$.} The expression for any element of the Clarke subgradient then follows from~\eqref{eq:def_clarke}. 

However, for an arbitrary conservative set-valued field, Assumption~\ref{hyp:conserv} does not necessarily hold. For instance, $D(x) = \mathds{1}(x \in \mathbb{N})$ is a conservative set-valued field for $p \equiv 0$, which does not satisfy Assumption~\ref{hyp:conserv}. Nevertheless, in practice, conservative set-valued fields naturally arise from a formal application of the chain rule. For a non-smooth but homogeneous activation function $\sigma$, one selects an element $e \in \partial \sigma (0)$, and computes $\sa_i(w)$ via backpropagation. Whenever a gradient candidate of $\sigma$ at zero is required (i.e., in~\eqref{eq:NN}, for some $j$, $V_j(W_j)$ contains a zero entry), it is replaced by $e$. 
Since $V_j(W_j)$ and $V_j(\lambda W_j)$ have the same zero elements, it follows that for every such $w$, $
\sa_i(\lambda w) = \lambda^L \sa_i(w)$. The conservative set-valued field $D_i$ is then obtained by associating to each $w$ the set of all possible outcomes of the chain rule, with $e$ ranging over all elements of $\partial \sigma(0)$. Thus, for such fields, Assumption~\ref{hyp:conserv} holds.

\paragraph{\edi{On the event $\cE = [\liminf \sm(u_k) >0]$.}} Training typically continues even after the training error reaches zero. \edi{The event $\cE$} characterizes this late-training phase, where our result applies. 
As noted earlier, since $\sm$ is $L$-homogeneous, the classification rule is determined by the direction of the  iterates $u_k=w_k/\norm{w_k}$. \edi{The event $\cE$ then captures the fact that}, beyond some iteration, the normalized margin remains positive. 
\edi{Studying SGD on $\cE$} is natural in the context of studying the implicit bias of SGD: we \emph{assume} that we reached the phase in which the dataset is correctly classified and \emph{then} characterize the limit points. A similar perspective was taken in  \cite{nacson2019lexicographic}, where the implicit bias of GF was analyzed under the assumption that the sequence of directions and the loss converge. However, unlike their approach, ours does not require assuming such convergence a priori.

Earlier works such as \cite{ji2020directional,Lyu_Li_maxmargin}, which analyze subgradient flow or smooth GD, establish convergence by assuming the existence of a single iterate $w_{k_0}$ satisfying $\sm(w_{k_0}) > \varepsilon$ and then proving that $\lim \sm(u_{k}) > 0$. Their approach relies on constructing a smooth approximation of the margin, which increases during training, ensuring that $\sm(u_k) > 0$ for all iterates with $k \geq k_0$. This is feasible in their setting, as they study either subgradient flow or GD with smooth $p_i$’s, allowing them to leverage the descent lemma.

In contrast, our analysis considers a nonsmooth and stochastic setting, in which, even if an iterate $w_{k_0}$ satisfying $\sm(w_{k_0}) > \varepsilon$ exists, there is no a priori assurance that subsequent iterates remain in the region where \edi{$\sm(u_k) >0$}. From this perspective, \edi{studying SGD on $\cE$} can be viewed as a stability assumption, ensuring that iterates continue to classify the dataset correctly. Establishing stability for stochastic and nonsmooth algorithms is notoriously hard, and only partial results in restrictive settings exist \cite{borkar2000ode,ramaswamy2017generalization,josz2024global}.



\section{Main result}\label{sec:main}

As a first step toward our main result, we establish that the iterates norm $(\lVert w_k\rVert)$ grows to infinity at a logarithmic rate. This is consistent with \cite[Theorem 4.3]{Lyu_Li_maxmargin}.
\begin{proposition}\label{prop:log_wk}
  Under Assumptions~\ref{hyp:loss_exp_log}--\ref{hyp:conserv}, \edi{on the event $\cE$}, there exist $c_1, c_2, \varepsilon>0$ and $k_0 \in \bbN$, such that for all $k \geq k_0$, $\norm{w_k}$ increases and
  \begin{equation*}
    c_1 \log (k)\leq \norm{w_k}^L \leq c_2 \log(k) \quad \textrm{ and } \quad 0 < \cL(w_k) \leq k^{-\varepsilon c_1}\, .
  \end{equation*}
  In particular, $\norm{w_k}\rightarrow + \infty$ and $\cL(w_k) \rightarrow 0$.
\end{proposition}
\begin{proof}[Sketch, full proof in Appendix~\ref{sec:pf_logwk}.] The proof follows from the next observations, which hold 
\edi{on the event $\cE$}, for $k$ large enough. \emph{(i)} There is $\varepsilon >0$, such that $\sm(u_k) \geq \varepsilon$. In particular, there are $M, C_1, C_2>0$, such that $ C_2 e^{-M \norm{w_k}^L}\leq -l'(p_i(w_k)) \leq C_1e^{-\varepsilon \norm{w_k}^L}$. \emph{(ii)} As a result, there is $C_3 >0$ such that
  \begin{equation}\label{eq:lwb_wk2}
    \norm{w_{k+1}}^2 \geq \norm{w_k}^2 ( 1 + C_3 \gamma e^{-M \norm{w_k}^L } \norm{w_k}^{L-2})\, ,
  \end{equation}
  which implies that $\norm{w_k}$ is increasing to infinity and that there is $C_4 >0$ such that
  \begin{equation}\label{eq:uwb_wk2}
    \norm{w_{k+1}}^2 \leq \norm{w_k}^2 (1 + C_4 \gamma e^{-\varepsilon \norm{w_k}^L}\norm{w_k}^{L-2})\, .
  \end{equation}\emph{(iii)} Finally, using~\eqref{eq:lwb_wk2}--\eqref{eq:uwb_wk2} and the Taylor's expansion of $(1+x)^{L/2}$ near zero, we obtain existence of constants $C_5, C_6, a>0$, such that for $k$ large enough,
  \begin{equation*}
    C_5 \gamma \leq e^{a \norm{w_k}^L}\left(\norm{w_{k+1}}^L - \norm{w_k}^L\right) \leq C_6 \gamma \, .
  \end{equation*}
Summing these inequalities from $k$ to $k+N$ and noticing that the expression in the middle is comparable to the integral of $e^{at}$ between $\norm{w_{k}}^L$ and $\norm{w_{k+N}}^L$, concludes the proof.
\end{proof}
Define the set-valued map $\bar{D} : \bbR^d \rightrightarrows \bbR^d$ as 
\begin{equation}\label{eq:avg_consfiel}
  \bar{D}(w) = \conv \{v: v \in D_i(w) \, , \textrm{with $i \in I(w)$} \}\,, \quad \textrm{  where $I(w) = \{ i: p_i(w) = \sm(w)\}$}\, .
\end{equation}
As shown in Appendix~\ref{app:conserv}, it is a conservative set-valued field for the potential $\sm$. Note that, following Remark~\ref{rmk:max_subg} \edi{in Appendix~\ref{app:conserv}}, even if for all $i$, $D_i = \partial p_i$, $\bar{D}$ can be different from $\partial \sm$. Next, we define the set-valued field $\bar{D}_s: \bbS^{d-1} \rightrightarrows \bbR^d$ as
\begin{equation}\label{def:riem_cons}
\bar{D}_{s}(u) := \{ v - \scalarp{v}{u}u : v \in \bar{D}(u) \}\, .
\end{equation}
The associated set of critical points is then given by
\begin{equation}\label{def:riem_crit}
  \cZ_s := \{ u \in \bbS^{d-1} : 0 \in \bar{D}_s(u) \} \subset \bbS^{d-1}\, .
\end{equation}
The field $\bar{D}_s$ and the critical points set $\cZ_s$ admit a straightforward interpretation. If $\sm$ is $C^1$ around some point $u \in \bbS^{d-1}$ and $\bar{D} =\{ \nabla \sm\}$, then $\bar{D}_s(u)$ is the \edi{transverse} component of $\nabla \sm(u)$, corresponding to its projection onto the tangent plane of $\bbS^{d-1}$ at $u$.
From a Riemannian geometry perspective, this implies that $\bar{D}_s(u)$ is the Riemannian gradient\footnote{Here, the Riemannian structure is implicitly induced from the ambient space.} of $\sm$ \emph{restricted} to the sphere $\bbS^{d-1}$, $\sm_{|\bbS^{d-1}}$. Similarly, $\cZ_s$ corresponds to the set of critical points of $\sm_{|\bbS^{d-1}}$.
More generally, since conservative fields are gradient-like objects (see Proposition~\ref{prop:var_strat_cons} in Appendix~\ref{app:omin}), we interpret $\bar{D}_s$ as the Riemannian conservative field of $\sm_{|\bbS^{d-1}}$, with $\cZ_s$ as its corresponding critical points\footnote{As noted in \cite[Page 4, footnote]{bolte2021conservative}, the concept of a conservative set-valued field extends naturally to functions defined on any complete Riemannian submanifold, including $\bbS^{d-1}$.}.

We will consider the differential inclusion (DI) associated to set-valued field $\bar{D}_s$,
\begin{equation*}\label{eq:DI_sphere}\tag{DI}
\dot{\su}(t) \in \bar{D}_{s}(\su(t))\,.
\end{equation*}
Under the aforementioned interpretation, this corresponds to the reversed gradient (or conservative field) flow of $\sm_{|\bbS^{d-1}}$. 

We now show that the iterates’ directions evolve according to a dynamic that approximates~\eqref{eq:DI_sphere} via an Euler-like discretization (or stochastic approximation). The proof is deferred to Section~\ref{pf:sto_app_explog}.
  \begin{proposition}\label{prop:stoch_approx_exp_log}
    Let \edi{Assumptions~\ref{hyp:loss_exp_log}--\ref{hyp:conserv}} hold. There exist \edi{three $\bbR^d$-valued } sequences $(\bg_k^s), (r_k), (\bar{\eta}_{k+1})$ \edi{and an $\bbR$-valued sequence $(\bgamma_k)$} such that, for both~\eqref{eq:gd_new} and~\eqref{eq:sgd_new}, the normalized direction iterates $u_k \coloneqq w_k/\lVert w_k \rVert$ satisfy
    \begin{equation}\label{eq:stoch_app_u}
      u_{k+1} = u_k + \bgamma_k\bg_k^s + \bgamma_k \bar{\eta}_{k+1} + \bgamma_k^2 r_k\, .
    \end{equation}
    Moreover, considering the filtration $(\cF_k)_k$ where, for $k \in \mathbb{N}$, $\cF_k$ the sigma-algebra generated by $\{ w_0,\ldots,w_k\}$, the following holds:
    \begin{enumerate}
      \item\label{pr_res:rk} \edi{On the event $\cE$}, the sequence $(r_k)$ satisfies $\sup_{k}\norm{r_k} < + \infty$.
      \item\label{pr_res:gammak} The sequence $(\bgamma_k)$ is positive and adapted to $(\cF_k)$. Moreover, \edi{on the event $\cE$}, $\sum_{k} \bgamma_k = + \infty$, and there is $c_3>0$ such that for sufficiently large $k$, $\bgamma_k \leq k^{-c_3}$.
      \item\label{pr_res:etak} For~\eqref{eq:gd_new}, $\bar{\eta}_{k} \equiv 0$. Otherwise, the sequence $(\bar{\eta}_{k})$ is adapted to $(\cF_k)$ and satisfies 
      \begin{equation*}
      \bbE[\bar{\eta}_{k+1} |\cF_k] = 0 \,.
      \end{equation*}
      Additionally, there exists a deterministic constant $c_4>0$ such that $\sup_{k} \norm{\bar{\eta}_{k+1}} < c_4$.
      \item\label{pr_res:barD} \edi{On the event $\cE$}, for any unbounded sequence $(k_j)_j$, such that $u_{k_j} \to u \in \bbS^{d-1}$, it holds that $\dist(\bar{D}_s(u), \bg^s_{k_j}) \rightarrow 0$. 
    \end{enumerate}
  \end{proposition}
    Since Proposition~\ref{prop:stoch_approx_exp_log} allows us to interpret $(u_k)$ as a discretization of~\eqref{eq:DI_sphere}, it is natural to investigate the convergence properties of a solution of its continuous counterpart~\eqref{eq:DI_sphere}.
  If $\su$ is such solution, then for almost every $t \in \bbR$, there exists $v \in \bar{D}_s(\su(t))$ such that $\dot{\su}(t) = v - \scalarp{v}{u}u$. Thus, by Definition~\ref{def:cons_f}, for almost every $t$, $\frac{\dif }{\dif t}  \sm(\su(t)) = \scalarp{\dot{\su}(t)}{v} = \norm{\dot{\su}(t)}^2$. Therefore, for $T >0$, we obtain
  \begin{equation*}
    \sm(\su(T)) - \sm(\su(0)) = \int_{0}^{T} \norm{\dot{\su}(t)}^2 \dif t \, .
  \end{equation*}
  This implies that $\sm(\su(T)) \geq \sm(\su(0))$, with strict inequality whenever $\su(0) \not \in \cZ_s$. In dynamical systems terminology, $-\sm$ is a Lyapunov function for~\eqref{eq:DI_sphere}. In particular, it can be shown that any solution $\su(t)$ to~\eqref{eq:DI_sphere} converges to $\cZ_s$.

Our main result, Theorem~\ref{thm:main}, establishes that the same holds true for the sequence of normalized directions: any limit point of $(u_k)$ is contained in $\cZ_s$. \edi{Since the sequence $(u_k)$ is bounded, it further implies that the entire sequence converges to the set of critical points $\cZ_s$.}
As discussed below, this result generalizes \cite[Theorem 4.4]{Lyu_Li_maxmargin} to stochastic gradient descent in the nonsmooth setting.

\begin{theorem}\label{thm:main}
  Under \edi{Assumptions~\ref{hyp:loss_exp_log}--\ref{hyp:conserv}}, \edi{on the event $\cE$}, $\sm(u_k)$ converges to a positive limit and 
  \begin{equation}\label{eq:conv_uk}
    \dist(u_k, \cZ_s) \xrightarrow[k \rightarrow + \infty]{} 0 \, .
  \end{equation}
\end{theorem}
\begin{proof}
  The proof, which is given in Appendix~\ref{pf:main_th} follows from Proposition~\ref{prop:stoch_approx_exp_log} and some minor adaptations of recent results on stochastic approximation from \cite{benaim2006dynamics,dav-dru-kak-lee-19}.
\end{proof}

A natural question is the interpretation of membership in $\cZ_s$. Given the Riemannian perspective on $\bar{D}_s$, it is unsurprising that belonging to $\cZ_s$ is a necessary optimality condition for the max-margin problem. We formally prove this result in Appendix~\ref{app:conserv}.
\begin{lemma}\label{lm:loc_max}
  If $u^*$ is a local maximum of $\sm_{|\bbS^{d-1}}$, then $0 \in \bar{D}_s(u^*)$.
\end{lemma}
Thus, Theorem~\ref{thm:main} establishes  that $(u_k)$ converge to the set of $\bar{D}_s$-critical points, which is a necessary condition of optimality for $\argmax_{u \in \bbS^{d-1}} \sm(u)$.
Comparing our result with \cite[Theorem 4.4]{Lyu_Li_maxmargin}, we note that, if each $D_i$ were equal to $\partial p_i$, then any limit point of $(u_k)$ would correspond \emph{exactly} to a scaled KKT point from \cite{Lyu_Li_maxmargin}. In this work, the authors formulate an alternative optimization problem, namely
\begin{equation*}\label{def:prob2}\tag{P}
  \min \{ \norm{w}^2 : w \in\bbR^d\, ,\sm(w) \geq 1\}\, .
\end{equation*}
 As discussed in \cite{Lyu_Li_maxmargin}, if there exists $w \in \bbR^d$ such that $\sm(w) >0$, solving~\eqref{def:prob2} is equivalent to maximizing the margin. Examining the KKT conditions (see Appendix~\ref{app:gen_sett}) of~\eqref{def:prob2}, we observe that for any $u \in \cZ_s$ such that $\sm(u)>0$, there exists $\lambda >0$, such that $\lambda u$ is a KKT point. This implies that, within the setting of Theorem~\ref{thm:main}, the  optimality characterization is \emph{identical} to that in \cite{Lyu_Li_maxmargin}. 
 
 These observations also highlight that the appearance of a conservative field in our problem is unrelated to backpropagation. The set $\cZ_s$ (thus, implicitly, $\bar{D}_s$) already arises in the analysis of continuous-time subgradient flow in \cite{Lyu_Li_maxmargin}. In fact, as previously noted, $\bar{D} \neq \partial \sm$, even if all $D_i = \partial p_i$ (see Remark~\ref{rmk:max_subg}).

However, we adopt a different perspective. Rather than linking $\cZ_s$ directly to the KKT points of~\eqref{def:prob2}, we interpret it as the set of $\bar{D}_s$-critical points of the margin restricted to the sphere—where $\sm$ is naturally defined due to homogeneity.
Moreover, our definitions of $\bar{D}_s$ and $\cZ_s$ remain valid even when the $D_i$’s are arbitrary conservative set-valued fields, not just subgradients. In fact, our stochastic approximation interpretation allows us to consider a more general setting (see Appendix~\ref{app:gen_sett}), where \edi{in which the event $\cE$ can be replaced by a larger event}. In this broader framework, the limit points of $(u_k)$ still lie in $\cZ_s$ without necessarily being rescaled versions of the KKT points of~\eqref{def:prob2}.

Finally, we note that as long as \edi{we are on ``stability event'' $\cE$}, our analysis allows the step-size \( \gamma \) to be of arbitrary  size. This may seem surprising, as (non-smooth) SGD typically requires vanishing step-sizes for convergence (\cite{majewski2018analysis,dav-dru-kak-lee-19,bolte2023subgradient,le2024nonsmooth}). Mathematically, this follows from the fact that \( \bar{\gamma}_k \), the \emph{effective} step-size of the dynamics, is actually decreasing in our setting. A convergence analysis of constant-step SGD for \emph{smooth} homogeneous linear classifiers was studied in \cite{nacson2019stochastic}, but to the best of our knowledge, the more general non-smooth setting had not yet been addressed.

\section{Proof of Proposition~\ref{prop:stoch_approx_exp_log}}\label{pf:sto_app_explog}
In this proof, $C, C_1, C_2, \ldots $ will denote some positive absolute constants that can change from equation to equation. We also note that for all $w,i$, $p_i(w) \leq C \norm{w}^L$ and, due to Assumption~\ref{hyp:conserv}, for all $v \in D_i(w)$, $\norm{v} \leq C \norm{w}^{L-1}$.

To obtain~\eqref{eq:stoch_app_u}, we appropriately rescale the step-size and then write the Taylor's expansion of $u \mapsto (u+h)/\norm{u + h}$, for small $h$, using the fact that $\norm{w_k} \rightarrow + \infty$. 

Towards that goal, let us first introduce the (stochastic) sequence
\begin{equation}\label{eq:noise}
  \eta_{k+1} := {\frac{n_b - n}{n_b n}} \sum_{i \in B_k} l'(p_i(w_k)) \sa_i(w_k) + \frac{1}{n}\sum_{i \notin B_k} l'(p_i(w_k)) \sa_i(w_k)\, .
\end{equation}
Both~\eqref{eq:gd_new} and \eqref{eq:sgd_new} (where for~\eqref{eq:gd_new}, $\eta_{k} \equiv 0$) can be rewritten as 
\begin{equation}\label{eq:sgd_noise}
  w_{k+1} = w_k - \frac{\gamma}{n}  \sum_{i=1}^n l'(p_i(w_k))\sa_i(w_k)+ \gamma \eta_{k+1}\, .
\end{equation}
Now let us introduce the following notations 
\begin{equation}\label{def:reparm_gamma}
  \tgamma_k = -\gamma \norm{w_k}^{L-1} \sum_{j=1}^n l'(p_j(w_k)) \, , \quad \bgamma_k = \tgamma_k\norm{w_k}^{-1}
\end{equation}
and 
\begin{equation}\label{eq:def_lmk_tilel}
  \lambda_{i,k} = \frac{l'(p_i(w_k))}{\sum_{j=1}^n l'(p_j(w_k))} \, , \quad \tilde{\eta}_{k+1} = \frac{-\eta_{k+1}}{\norm{w_k}^{L-1} \sum_{j=1}^n l'(p_j(w_k))}\, .
\end{equation}
Note that since $l'(q) <0$, for all $k$,  $\tgamma_k, \bgamma_k, \lambda_{i,k} \geq 0$. Moreover, $\sum_{i=1}^n\lambda_{i,k} = 1$.

By Assumption~\ref{hyp:conserv} for each $i,k$, it holds that $v_{i,k} = \norm{w_k}^{L-1}g_{i,k}$, where $g_{i,k} \in D_i(u_k)$. Thus, we can rewrite Equation~\eqref{eq:sgd_noise} as:
\begin{equation}\label{eq:first_wk}
    w_{k+1} = w_k + \tilde{\gamma}_k \sum_{i=1}^n \lambda_{i,k} g_{i,k} + \tilde{\gamma}_k \tilde{\eta}_{k+1} :=w_{k} + \tilde{\gamma}_k \bar{g}_k + \tgamma_k \tilde{\eta}_{k+1} \, ,
\end{equation}
with $\bar{g}_k = \sum_{i=1}^n\lambda_{i,k} g_{i,k}$.
Therefore, using the definition of $\bgamma_k$ is~\eqref{def:reparm_gamma},
\begin{equation}\label{eq:uk_taylor1}
  \begin{split}
      u_{k+1} &= \frac{w_{k}+\tgamma_k \bg_k + \tgamma_k \tilde{\eta}_{k+1}}{\norm{w_{k}+\tgamma_k \bg_k + \tgamma_k \tilde{\eta}_{k+1}}} =
      \frac{u_k + \bgamma_k \bg_k + \bgamma_k \tilde{\eta}_{k+1}}{\norm{u_k + \bgamma_k \bg_k + \bgamma_k \tilde{\eta}_{k+1}}}\, .
        \end{split}
\end{equation}
Using the Taylor's expansion $\norm{u+h}^{-1} = 1- \scalarp{u}{h} + \cO(\norm{h}^2)$, for $u \in \bbS^{d-1}$, we obtain
\begin{equation}\label{eq:pf_uk_stochapp}
  u_{k+1}=(u_k + \bgamma_k \bg_k + \bgamma_k \tilde{\eta}_{k+1})(1 - \bgamma_k\scalarp{\bg_k}{u_k} -\bgamma_k\scalarp{\tilde{\eta}_{k+1}}{u_k} + b_k) = u_k + \bar{\gamma}_k \bg_k^s + \bgamma_k \bar{\eta}_{k+1} + \bar{\gamma}_k^2 r_k\, ,
\end{equation}
where $b_k$ is such that $|b_k| \leq C \bgamma_k^2 (\norm{\bg_k} + \norm{\tilde{\eta}_{k+1}})^2$, as soon as $\bgamma_k \norm{\bg_k + \tilde{\eta}_{k+1}} \leq 1/2$, and 
\begin{equation}\label{eqdef:proj_nois_subg}
  \bar{\eta}_{k+1} := \tilde{\eta}_{k+1} - \scalarp{\tilde{\eta}_{k+1}}{u_k} u_k \quad \textrm{ and } \quad \bar{g}_k^s := \bg_k - \scalarp{\bg_k}{u_k} u_k\, ,
\end{equation}
and, finally, 
\begin{equation}\label{eq:rk_bound}
 \bgamma_k^2 \norm{r_k} \leq |b_k|\norm{(u_k + \bgamma_k \bg_k + \bgamma_k \tilde{\eta}_{k+1})} + \bgamma_k^2 \norm{(\bg_k + \tilde{\eta}_{k+1})\scalarp{\bg_k + \tilde{\eta}_{k+1}}{u_k}}\, .
\end{equation}
Equation~\eqref{eq:pf_uk_stochapp} correspond to~\eqref{eq:stoch_app_u}. We now briefly prove the four points of the proposition. 

\emph{Claim on $(\bar{\eta}_k)$.} The fact that $\eta_{k}$ and thus $\bar{\eta}_k$ is $w_k$-measurable is immediate by its definition in~\eqref{eq:noise}. Additionally, $\bbE[\bar{\eta}^s_{k+1} |\cF_k] = \bbE[\eta_{k+1} |\cF_k] = 0$. Moreover, Assumption~\ref{hyp:conserv} and Equation~\eqref{eq:noise} implies
\begin{equation*}
\norm{\eta_{k+1}} \leq C_1\norm{\sum_{i=1}^n l'(p_i(w_k))\sa_i(w_k)} \leq C_2 \norm{w_k}^{L-1} \sum_{i=1}^n |l'(p_i(w_k))|\, ,
\end{equation*}
with $C_2>0$ some deterministic constant independent on $k$. Therefore, by~\eqref{eq:def_lmk_tilel} and \eqref{eqdef:proj_nois_subg}, $\norm{\bar{\eta}_{k+1}} \leq \norm{\tilde{\eta}_{k+1}} \leq C_2$.

\emph{Claim on $(\bgamma_k)$.} Almost surely \edi{on $\cE$}, there is $\varepsilon >0$, such that for $k$ large enough, $\sm(u_k) \geq  \varepsilon$. Thus, for every $i$, $l'(p_i(w_k)) \leq e^{-\varepsilon\norm{w_k}^L }$. By Assumption~\ref{hyp:conserv}, $\norm{\sa_i(w_k)} \leq C \norm{w_k}^{L-1}$, which implies for $k$ large enough,
\begin{equation*}
  \bgamma_k \leq C_1 \norm{w_k}^{L-2} e^{-\varepsilon \norm{w_k}^L} \leq \frac{C_1 c_2 \log(k)^{L-2}}{k^{\varepsilon c_1}} \leq \frac{1}{\sqrt{k^{\varepsilon c_1}}}
\end{equation*}
where the penultimate inequality comes from Proposition~\ref{prop:log_wk}.
To show that $\sum_{k} \bgamma_k = + \infty$, note that using Equation~\eqref{eq:first_wk}, \edi{on the event $\cE$}, we have for $k$ large enough,
\begin{equation}\label{eq:sum_gamma_inf}
  \begin{split}
\norm{w_{k+1}}^2  &\leq \lVert w_{k} \rVert^2 + 2 \tilde{\gamma}_k \lVert w_{k} \rVert \lVert \bg_k + \tilde{\eta}_{k+1} \rVert  + \tgamma_k^2 \norm{\bg_k + \tilde{\eta}_{k+1}}^2\\
&\leq \norm{w_k}^2 (1 + C \bgamma_k + C_1 \bgamma_k^2) \leq \norm{w_k}^2 e^{C_2 \bgamma_k} \leq \norm{w_{k_0}}^2 e^{C_2 \sum_{i=k_0}^{k} \bgamma_i }
\end{split}
\end{equation}
where $k_0$ is large enough, and where we have used the fact that $\sup_{k} (\norm{\bg_k + \tilde{\eta}_{k+1}}) < + \infty$ and that $\bgamma_k \rightarrow 0$ \edi{on $\cE$}. 
Since the left-hand side of Equation~\eqref{eq:sum_gamma_inf} goes to infinity \edi{on $\cE$} by Proposition~\ref{prop:log_wk}, we obtain that the right-hand side diverge to infinity and therefore, \edi{on $\cE$,} $\sum_{k } \bgamma_k = + \infty$.

\emph{Claim on $(r_k)$.} Since \edi{on $\cE$}, $\bgamma_k \rightarrow 0$, $\sup_{k} \norm{\tilde{\eta}_{k+1}} \leq C_1$ and $\sup_{k} \norm{\bg_k} < C_2$, there is $k_0$, such that for all $k \geq k_0$, $\bgamma_k (\norm{\bg_k} + \norm{\tilde{\eta}_{k+1}})\leq 1/2$. Therefore, for $k \geq k_0$, $|b_k| \leq C \gamma_k^2 \norm{\bg_k + \tilde{\eta}_{k+1}}$ in~\eqref{eq:pf_uk_stochapp}, which by~\eqref{eq:rk_bound} implies that $\sup_{k \geq k_0}\norm{r_k} \leq C$.

\emph{Claim on $\bar{D}$.} Consider a sequence $u_{k_j} \rightarrow u$ and $g$ any accumulation point $g_{k_j}$, we need to prove that $g - \scalarp{g}{u}u \in \bar{D}_s(u)$, or, equivalently, that $g \in \bar{D}(u)$. Recall that $\bg_k = \sum_{i=1}^k \lambda_{i,k} \bg_{i,k}$, where $\lambda_{i,k} \geq 0$, $\sum_{i} \lambda_{i,k} = 1$ and $g_{i,k} \in D_{i}(u_k)$. Extracting a subsequence, we can assume that for each $i$, $\lambda_{i, k_j} \rightarrow \lambda_i$ and $g_{i,k_j} \rightarrow g_i \in D_i(u)$.  We claim that $\lambda_i \neq 0 \implies p_i(u) = \sm(u)$. Indeed, without losing generality assume that $p_1(u_{k_j})\rightarrow \sm(u)$. Then, if $\lim (p_{i}(u_k) - \sm(u)) >0$, \edi{on $\cE$} we obtain 
\begin{equation*}
  \lambda_{i,k_j} \leq \frac{l'(\norm{w_{k,j}}^L p_i(u_{k_j}))}{l'(\norm{w_{k_j}}^L p_1(u_{k_j}))} \leq C e^{-\norm{w_k}^L\left( p_i(u_k) - \sm(u)\right)}\xrightarrow[j \rightarrow + \infty]{} 0 \, .
\end{equation*}
Therefore, $g$ can be written as $ \sum_{i=1}^n \lambda_i g_i$, where $g_i \in D_i(u)$ and $\lambda_i \neq 0 \implies p_i(u) = \sm(u)$.
In other words, $g \in \bar{D}(u)$, concluding the proof. 

\hfill $\blacksquare$

\section*{Acknowledgement}

The authors thank Evgenii Chzhen for insightful discussions.

\bibliography{main}

\appendix
\appendix

\newpage

\section{o-minimal structures}\label{app:omin}

In every statement of this paper \emph{semialgebraic} can be replaced by \emph{definable} in a fixed \emph{o-minimal structure}. We collect here some useful facts about definable sets and maps. In particular, Proposition~\ref{prop:var_strat_cons}, that provides a variational description of a definable conservative set-valued field of a definable potential, will be used in Appendix~\ref{pf:main_th} to prove Theorem~\ref{thm:main}.

For more details on o-minimal structure we refer to the monographs \cite{cos02,van1998tame,van96}. A nice review of their importance in optimization is \cite{iof08}.

The definition of an o-minimal structure is inspired by properties that are satisfied by semialgebraic sets.
\begin{definition}
  We say that $\cO:=(\cO_n)$, where for each $n \in \bbN$, $\cO_n$ is a collection of sets in $\bbR^n$, is an o-minimal structure if the following holds. 
  \begin{enumerate}[label=\roman*)]
    \item If $Q: \bbR^n \rightarrow \bbR$ is a polynomial, then $\{x \in \bbR^n : Q(x) = 0 \} \in \cO_n$.
    \item For each $n \in \bbN$, $\cO_n$ is a boolean algebra: if $A, B \in \cO_n$, then $A \cup B, A \cap B$ and $A^c$ are in $\cO_n$.
    \item If $A \in \cO_n$ and $B \in \cO_m$, then $A \times B \in \cO_{n+m}$.
    \item If $A \in \cO_{n+1}$, then the projection of $A$ onto its first $n$ coordinates is in $\cO_n$.
    \item Every element of $\cO_1$ is exactly a finite union of intervals and points of $\bbR$. 
  \end{enumerate}
\end{definition}

Sets contained in $\cO$ are called \emph{definable}. We call a map
$f : \bbR^d \rightarrow \bbR^m$ definable if its graph is definable. Similarly, $D: \bbR^d \rightrightarrows \bbR^d$ is definable if $\Graph D = \{(w,v): v \in D(w) \}$ is definable.
Definable sets and maps have remarkable stability
properties. For instance, if $f$ and $A$ are definable, then $f(A)$
and $f^{-1}(A)$ are definable and definability is stable by most of the common operators such as  $\{+, -, \times, \circ, \circ^{-1}\}$. Let us look at some examples of o-minimal structures.

\textbf{Semialgebraic.} Semialgebraic sets form an o-minimal structure. This follows from the celebrated result of Tarski \cite{tarski1951decision}.

\textbf{Globally subanalytic.} There is an o-minimal structure that contains, for every $n \in \bbN$, sets of the form $\{ (x,t) : t = f(x)\}$, where $f : [-1, 1]^n \rightarrow \bbR$ is an analytic function. This comes from the fact that subanalytic sets are stable by projection, which was established by Gabrielov \cite{gabrielov1968projections, gabrielov1996complements}. The sets belonging to this structure are called globally subanalytic (see \cite{bier_semi_sub} for more details).\\
\textbf{Log-exp.} There is an o-minimal structure that contains, semialgebraic sets, globally sub-analytic sets as well as the graph of the exponential
and the logarithm (see \cite{wilkie1996model, van1994elementary}).

With these examples in mind it is usually easy to verify that a function is definable. This will be the case as soon as the function is constructed by a finite number of definable operations on definable functions. From this, we see that \emph{most of neural networks architectures} are definable in the structure \emph{Log-exp}.

Let us record here a striking fact. The domain of a definable function can be always partitioned in manifolds such that $f$ is differentiable on each set of this partition. Before stating this result we briefly introduce to the reader some basics of differential geometry. We refer to the monographs \cite{Lafontaine_2015,boumal2023intromanifolds} for a more detailed introduction on these notions.

\paragraph{Submanifolds, tangent spaces and Riemannian gradients.}

A set $\cX \subset \bbR^d$ is said to be a $k$-dimensional $C^p$-manifold if for every $x \in \cX$, there is a neighborhood $\cU \subset \bbR^d$ of $x$ and a $C^p$ function $g: \cU \rightarrow \bbR^{d-k}$ such that the Jacobian of $g$ at every point of $\cU$ is of full rank. The tangent space of $\cX$ at $x$ is $\cT_{\cX}(x) =\ker \Jac g(x)$. Equivalently, it is the set of vectors $v \in \bbR^d$ such that there is $\varepsilon >0$ and a $C^1$ curve $\su : (- \varepsilon, \varepsilon) \rightarrow \cX$, such that $(\su(0), \dot{\su}(0)) = (x,v)$.

We say that a function $f: \cX \rightarrow \bbR$ is $C^p$ if around every $x$ there is a neighborhood $\cU \subset \bbR^d$ and a $C^p$ function $\tilde{f} : \cU\rightarrow \bbR$ such that $\tilde{f}_{|\cU} = f_{|\cU}$. For such $f$ and $x \in \cX$, the Riemannian gradient (here, we implicitly induce the Riemannian structure from the ambient space) of $f$ at $x$, is
\begin{equation}\label{eqdef:riem_grad}
\nabla_{\cX}f(x) :=   P_{\cT_{\cX}(x)} \nabla \tilde{f}(x)\, ,
\end{equation} 
where $P_{\cT_{\cX}(x)}$ is the orthogonal projection onto $\cT_{\cX}(x)$. This definition is independent of the smooth representative $\tilde{f}$.

For the following two propositions we fix a definable structure $\cO$ (e.g., semialgebraic sets) and definable will mean definable in $\cO$.
\begin{proposition}[{\cite[4.8]{van1998tame}}]\label{prop:def_strat}
  Consider  a definable $f: \bbR^d \rightarrow \bbR$. For every $p>0$, there is $(\cX_i)_{1 \leq i \leq l}$ a finite partition \footnote{This partition also satisfies some important properties and is usually called a stratification.} of $\bbR^d$ into $C^p$ submanifolds such that for every $i$, $f_{|\cX_i}$ is $C^p$, in the sense of differential geometry. Moreover, if $\cM \subset \bbR^d$ is definable, then this partition can be chosen such that there is $\cI \subset \{1, \ldots, l \}$ for which $\cM = \bigcup_{i \in \cI} \cX_i$.
\end{proposition}
Furthermore, the partition can be chosen in a way that along the tangent directions $D(w)$ is simply the Riemannian gradient of $f$. This will be useful in our proof of Theorem~\ref{thm:main}.
\begin{proposition}[{\cite[Theorem 4]{bolte2021conservative}}]\label{prop:var_strat_cons}
  Consider a definable, locally Lipschitz continuous $f: \bbR^d \rightarrow \bbR$ a potential of a definable conservative set-valued field $D: \bbR^d \rightrightarrows \bbR^d$. For any $p >0$, there is a partition $(\cX_i)_{1 \leq i \leq l}$ such that the result of Proposition~\ref{prop:def_strat} holds and, moreover, for any $w \in \cX_i$, 
  \begin{equation*}
    P_{\cT_{\cX_i}(w)} D(w) = \{\nabla_{\cX_i} f(w)\}\footnote{Note the similarity with~\eqref{eqdef:riem_grad}.}\, .
  \end{equation*}
  Moreover, if $\cM \subset \bbR^d$ is definable, then this partition can be chosen such that there is $\cI \subset \{1, \ldots, l \}$, for which $\cM = \bigcup_{i \in \cI} \cX_i$.
\end{proposition}

\begin{remark}
  The last point of Proposition~\ref{prop:var_strat_cons} is not formally stated in \cite[Theorem 4]{bolte2021conservative}. Nevertheless, in its proof, the first stage of the construction of $(\cX_i)$ is based on Proposition~\ref{prop:def_strat}. Thus, at this stage we can use the last point of Proposition~\ref{prop:def_strat} (i.e. the one of \cite[4.8]{van1998tame}) and obtain the stated result.
\end{remark}

\section{Conservative mappings}\label{app:conserv}

The purpose of this section is to prove the fact that \emph{i)}$\bar{D}$ from Equation~\eqref{def:riem_cons} is a conservative set-valued field of $\sm$, \emph{ii)} Lemma~\ref{lm:loc_max}.

First, we introduce a generalization of conservative fields to vector-valued functions.

\begin{definition}[\cite{bolte2021conservative}]\label{def:cons_map}
  For a locally Lipschitz function $f = (f_1, \ldots, f_m): \bbR^d \rightarrow \bbR^m$, a set-valued map $J: \bbR^d \rightrightarrows \bbR^{m \times d}$ is said to be a \emph{conservative mapping} for $f$ if for every $ 1\leq i \leq m$ the $i$-th row of $J$ is a conservative field for $f_i$. 
\end{definition}

One of the issue of the Clarke subgradient is that it is not closed under composition. That is to say, if $f, g : \bbR \rightarrow \bbR$ are locally Lipschitz we do not necessarily have $\partial (f\circ g) = \partial f \times \partial g$. On the contrast, product of conservative mapping remains conservative. This is the main reason why the backpropagation algorithm outputs an element of a conservative set-valued field.
\begin{proposition}[{\cite[Lemma 5]{bolte2021conservative}}]\label{pr:comp_cons}
  Consider $f : \bbR^d \rightarrow \bbR^m$, $g: \bbR^m \rightarrow \bbR^{l}$ two locally Lipschitz continuous functions with $D_f: \bbR^{d} \rightrightarrows \bbR^{m \times d}$, $D_g : \bbR^m \rightrightarrows \bbR^{l \times m}$ two corresponding conservative mappings. The set-valued map
  \begin{equation*}
    D_{g \circ f} : w \mapsto \{J \in \bbR^{l \times m} : \textrm{$J = J_2 J_1$, with $J_f \in D_f(w)$ and $J_g \in D_{g}(f(w))$}\}\, 
  \end{equation*} 
  is a conservative mapping of $g \circ f$.
\end{proposition}

 In the context of the paper, $p_1, \ldots, p_n : \bbR^d \rightarrow \bbR$ are semialgebraic potentials of semialgebraic conservative set-valued fields $D_1, \ldots, D_n$. Recall that $\sm(w) = \min_{i} p_i(w)$ and the definition $\bar{D}$ in~\eqref{eq:avg_consfiel}. We now have the tools to prove that $\bar{D}$ is a conservative set-valued field of $\sm$.

\begin{lemma}\label{lm:max_consgrad}
  The map $\bar{D}: \bbR^d \rightrightarrows \bbR^d$ is a semialgebraic conservative set-valued field for the potential $\sm$.
\end{lemma}
\begin{proof}
  The fact that $\Graph\bar{D}$ is semialgebraic comes from the fact that it is constructed from finite semialgebraic operations involving only semialgebraic sets. To prove that $\bar{D}$ is a conservative set-valued field note that $\sm$ can be written as a composition of semialgebraic functions:
  \begin{equation*}
    w \overset{\varphi_1}{\mapsto} (p_1(w), \ldots, p_n(w)) \overset{\varphi_2}{\mapsto} \min(p_1(w), \ldots, p_n(w))
  \end{equation*}
   Then, $D_p(w) = [D_1(w), \ldots, D_n(w)]^{\top}$ is a conservative mapping of $\varphi_1$ and $(x_1, \ldots, x_n) \rightrightarrows \conv\{ e_i : x_i = \min(x_1, \ldots, x_n) \}$, where $e_i \in \bbR^n$ is the $i$-the element of the canonical basis, is the Clarke subgradient (and thus a conservative gradient) of $\varphi_2$. The claim follows from Proposition~\ref{pr:comp_cons} and the fact that conservativity is stable by convexity.
\end{proof}

\begin{remark}\label{rmk:max_subg}
  Even if $D_1, \ldots, D_n$ are Clarke subgradients, $\bar{D}$ is not necessarily a Clarke subgradient. For instance, if $n=2$, $p_1(w) = \min(0,w)$ and $p_2(w) = \min(0, -w)$, then $\sm(w) = 0$ and for any $w \neq 0$, $\bar{D}(w) = \{0\}$. Nevertheless, $\bar{D}(0) = [-1,1]$. Thus, conservative set-valued fields appear naturally in the analysis of our problem, independently of the use of backpropagation.
\end{remark}

We finish this section by a proof of Lemma~\ref{lm:loc_max}. Note that 
$\bbS^{d-1} = \{ u \in\bbR^d: \norm{u} =1\} = g^{-1}(0)$, with $g(w) = \norm{w}^2$. Therefore, $\bbS^{d-1}$ is a $d-1$ dimensional $C^{\infty}$-manifold and for every $u \in \bbS^{d-1}$, 
\begin{equation}\label{eq:tang_sphere}
\cT_{\bbS^{d-1}}(u) = \{v - \scalarp{v}{u} : v \in \bbR^d \}\, .
\end{equation}

\begin{proof}[Proof of Lemma~\ref{lm:loc_max}]
  Assume the contrary and consider $v_s \in \argmin \{ \norm{v_s'}: v_s' \in \bar{D}_s(u^*)\}$. Since $v_s$ is in the tangent space of $\bbS^{d-1}$ at $u^*$, there is a $C^1$ curve $\su: (-\varepsilon, \varepsilon) \rightarrow \bbS^{d-1}$ such that $\su(0) = u^*$ and $\dot{\su}(0) = v_s$. Moreover, since $\bar{D}_s(u^*)$ is convex and $v_s$ is its element of minimal norm, for any $v_s'\in D(u^*)$, $\scalarp{v_s}{v_s'} \geq \norm{v_s}^2 >0$. 
  
  Since $\Graph \bar{D}_s$ is closed, it also holds, for $t \geq 0$ small enough, that for any $v_s' \in \bar{D}_s(\su(t))$, $2\scalarp{v_s}{v_s'} \geq \norm{v_s}^2$ and, thus, for $t$ small enough, $4\scalarp{v_s'}{\dot{\su}(t)} \geq \norm{v_s}^2$. 
  
  In particular, for almost every $t \in [0,\delta]$, for $\delta$ small enough, and every $v \in \bar{D}(\su(t))$, 
  \begin{equation*}
    \frac{\dif}{\dif t}\sm(\su(t)) = \scalarp{v}{\dot{\su}(t)} = \scalarp{ v - \scalarp{v}{\su(t)}\su(t)}{\dot{\su}(t)} \geq \frac{1}{4}\norm{v_s}^2 >0 \, ,
  \end{equation*}
  where we have used that $v - \scalarp{v}{\su(t)}\su(t) \in \bar{D}_s(\su(t))$ and that $\scalarp{\su(t)}{\dot{\su}(t)} = 0$.
  Therefore, $\sm$ strictly increase on $[0, \delta]$, which contradicts the fact that $u^*$ is a local maximum.
\end{proof}

\section{Discretization of differential inclusions}\label{app:interp}

\paragraph{Differential inclusions.}
Consider $B \subset \bbR^d$. We say that $\sH$ is a set-valued map from $B$ to $\bbR^d$, denoted $\sH: B \rightrightarrows \bbR^d$, if for every $w \in B$, $\sH(w)$ is a subset of $\bbR^d$. We say that $\sH$ is graph-closed if $\Graph H := \{(w,v): v \in \sH(w) \}$ is closed. It is said to have nonempty (respectively convex) values if for every $w \in B$, $\sH(w)$ is nonempty (respectively convex). It is said to be locally bounded, if every $w \in \bbR^d$ admits a neighborhood $\cU$ of $w$ and $M>0$, such that for every $w' \in \cU$, and $v \in \sH(w')$, $\norm{v} \leq M$.

\paragraph{Differential inclusions and Lyapunov functions.}
To every such $\sH$ we can associate a differential inclusion (DI):
\begin{equation}\label{eq:DI}
    \dot{\su}(t) \in \sH(\su(t)) \, .
\end{equation}
We say that $\su: \bbR_{+} \rightarrow B$ is a solution to~\eqref{eq:DI}, if Equation~\eqref{eq:DI} is satisfied for almost every $t \geq 0$.
A continuous function $\Lambda: B \rightarrow \bbR$ is said to be a Lyapunov function of~\eqref{eq:DI} for a set $\cZ \subset B$, if for every solution $\su$ of~\eqref{eq:DI},
\begin{equation*}
  \Lambda(\su(t)) \leq \Lambda(\su(0))\, ,
\end{equation*}
with strict inequality as soon as $\su(0) \not \in \cZ$. 

In the context of our paper, we are of course interested in the case where $B  = \bbS^{d-1}$ and $\sH = \bar{D}_s$. In such case, as explained before the statement of Theorem~\ref{thm:main}, $-\sm$ is a Lyapunov function for the set $\cZ_s$.

\paragraph{Discretization of differential inclusions.}
Consider a $\bbR^d$-valued sequence $(u_k)$, satisfying the following recursion:
\begin{equation*}
  u_{k+1} = u_k + \gamma_k v_k + \gamma_k e_k\, ,
\end{equation*}
where $(\gamma_k)$ is a sequence of positive stepsizes, where $v_k,e_k\in \bbR^d$. In practice, $v_k$ will be chosen close to an element of $\sH(u_k)$, with $\sH$ some set-valued map (in our setting $\sH = \bar{D}_s$), and $(e_k)$ will represent some (stochastic) perturbations that will vanish when $k \rightarrow + \infty$. Hence, one can view this algorithm as an Euler-like discretization of the DI associated to $\sH$. 

Under mild assumptions, under a presence of a Lyapunov function, we can characterize the accumulation points of $(u_k)$.
The first such result was established in the seminal work of \cite{benaim_05_DI_1}.
Similar statements were lately proved in \cite{duchi2018stochastic,borkar2008stochastic,dav-dru-kak-lee-19}. Here we use a version of the result proved in \cite[Theorem 3.2]{dav-dru-kak-lee-19}. 

To state the proposition, for every $T>0$ and $k \in \bbN$, we define
\begin{equation*}
  \sn(k, T) = \sup \left\{ l \geq k: \sum_{i=k}^l \gamma_i \leq T \right\}\, .
\end{equation*}

\begin{proposition}\label{pr:stoch_approx_our}
  Let $B \subset \bbR^d$ be closed and let $\sH: B \rightrightarrows \bbR^d$ be a graph-closed set-valued map, with convex and nonempty values. Assume the following.
  \begin{enumerate}[label=\roman*)]
    \item The sequences $(u_k), (v_k)$ are bounded and any accumulation point of $(u_k)$ lies in $B$.
    \item $(\gamma_k)$ is a sequence of positive numbers such that $\sum_{k} \gamma_k = +\infty$ and $\gamma_k \rightarrow 0$.
    \item\label{hyp:perturb_zero} For every $T>0$, 
    \begin{equation*}
     \lim_{k \rightarrow + \infty} \sup\left\{ \norm{\sum_{i=k}^l \gamma_i e_i} : k \leq l \leq \sn(k,T)\right\}  = 0 \, .
    \end{equation*}
    \item \label{hyp:drus_lyap} There is $\Lambda : B \rightarrow \bbR$, a Lyapunov function associated with~\eqref{eq:DI} and the set $\cZ = \{u \in B: 0 \in \sH(u)\}$, such that $\Lambda(\cZ)$ is of empty interior.
     \item\label{hyp:drus_conv} For any unbounded sequence $(k_j)$, such that $(u_{k_j})$ converges to some $u\in B$, it holds that 
    \begin{equation}\label{eq:conv_setv_seq}
     \lim_{j \rightarrow + \infty} \dist(\sH(u), v_{k_j}) =   0
    \end{equation}
   \end{enumerate}
Then the limit points of $(u_k)$ are in $\cZ$ and the sequence $\Lambda(u_k)$ converges.
\end{proposition}
The statement of \cite[Theorem 3.2]{dav-dru-kak-lee-19} have slightly different assumptions on $\sH$, $(e_k)$ and $(\gamma_k)$, and we devote the end of this section to describe how to adapt their proof to obtain our statement.

 The only difference between Proposition~\ref{pr:stoch_approx_our} and \cite[Theorem 3.2]{dav-dru-kak-lee-19} are the following.

\begin{enumerate}
  \item In \cite[Theorem 3.2]{dav-dru-kak-lee-19} there are no assumptions on $\sH$ and~\eqref{eq:conv_setv_seq} is replaced by 
  \begin{equation*}
   \lim_{N \rightarrow + \infty} \dist\left(\sH(u), \frac{1}{N} \sum_{j=1}^N v_{k_j}\right) = 0\, .
  \end{equation*}
  \item In \cite[Theorem 3.2]{dav-dru-kak-lee-19} there is a requirement that $\sum_k \gamma_k^2 < + \infty$ and that $\sum_{k}\gamma_k e_k$ converges.
\end{enumerate}

The only part in the proof of \cite[Theorem 3.2]{dav-dru-kak-lee-19} where these two points are needed are in their use of \cite[Theorem 3.1]{dav-dru-kak-lee-19} which corresponds to \cite[Theorem 3.7]{duchi2018stochastic}. To state the result let us first define $\tau_k = \sum_{i=0}^k \gamma_i$ and the linearly interpolated processes $\sU$ as 
 \begin{equation*}
  \sU(t) = u_k + \frac{t - \tau_k}{\gamma_{k+1}} (u_{k+1} - u_k) \quad \textrm{ if $t \in [\tau_k, \tau_{k+1}]$}\, .
 \end{equation*}
 \begin{proposition}[ {\cite[Theorem 4.2]{benaim_05_DI_1}\cite[Theorem 3.7]{duchi2018stochastic}}]\label{prop:interp_proc}
  \phantom{=}\\For any $T>0$ and any sequence $t_k \rightarrow + \infty$, we can extract a subsequence $(k_j)$ such that there is a continuous $\su : [0,T] \rightarrow \bbR^d$, for which, 
  \begin{equation*}
    \lim_{j \rightarrow + \infty}\sup_{h \in [0, T]} \norm{\sU(t_{k_j} +h) -\su(h)} = 0 \, , 
  \end{equation*}
  and $\su$ is a solution to~\eqref{eq:DI}.
 \end{proposition}
 This result holds under our assumptions. The idea of the proof is standard and goes back to \cite[Proof of Theorem 4.2]{benaim_05_DI_1}.
 
 Indeed, define piecewise constant processes $\sV, \sE : \bbR_{+} \rightarrow \bbR^d$ as
 \begin{equation*}
 \textrm{for $t \in [\tau_k, \tau_{k+1})$}\, , \quad  \sV(t) = v_k\, ,  \quad \sE(t) = e_{k}\, .
 \end{equation*}
Note that for any $t, h>0$,
 \begin{equation*}
  \sU(t+h) = \sU(t) + \int_{t}^{t+h} (\sV(t')+\sE(t')) \dif t'\, . 
 \end{equation*}
Moreover, from our assumptions on $(\gamma_k), (e_k)$, it holds that
\begin{equation*}
 \lim_{t \rightarrow + \infty} \norm{\sup_{h \in [0,T]} \int_{t}^{t+h} (\sE(t')) \dif t'} = 0 \, .
\end{equation*}
 Therefore, fixing $t_{k} \rightarrow + \infty$, and proceeding as in \cite[Proof of Theorem 4.2]{benaim_05_DI_1}, there is $\sv, \su : [0, T] \rightarrow + \infty$ and an extracted sequence $(k_j)$ such that
 \begin{equation*}
 \lim_{j \rightarrow + \infty} \sup_{h \in [0,T]}\norm{\sU(t_{k_j} + h) - \su(h)} = 0 \, ,
 \end{equation*}
 and for any $h \in [0,T]$,
 \begin{equation*}
  \su(h) = \su(0) + \int_{0}^h \sv(t') \dif t'\, ,
 \end{equation*}
 Moreover, for almost every $t' \in[0,T]$, there is a sequence $(\tilde{v}_{k_j})$ such that for each $j$, $\tilde{v}_{k_j}$ is a finite convex combination of elements in $\{\sV(t_{k_{j'}} + t'): j' \geq j \}$ and such that
\begin{equation*}
  \sv(t') = \lim_{j \rightarrow + \infty} \tilde{v}_{k_j} \in \sH(\su(t'))\, ,
\end{equation*}
where the last inclusion comes from~\eqref{hyp:drus_conv} and the fact that $\sH$ is convex-valued.
Consequently, we obtain for almost every $t' \in [0,T]$, $\dot{\su}(t') = \sv(t')\in  \sH(\su(t'))$, which shows the statement of Proposition~\ref{prop:interp_proc} and thus the one of Proposition~\ref{pr:stoch_approx_our}.

\section{Proof of Theorem~\ref{thm:main}}\label{pf:main_th}
To obtain Theorem~\ref{thm:main}, we only need to check the assumptions of Proposition~\ref{pr:stoch_approx_our} \edi{on the event $\cE$}, with $\sH = \bar{D}_s$, $\Lambda = -\sm$, $v_k = \bg_k^s$ and $e_k = \eta_{k+1} + \gamma_k r_k$.

\paragraph{Boundedness of $(u_k)$ and $(\bg_k^s)$.} For each $k$, $\norm{u_k} = 1$. The fact that $(\bg_k^s)$ is bounded comes from the fact that $\lim_{k}\dist(\bar{D}_s(\bbS^{d-1}), g_k^s) = 0$ (or we could already see it from the way $\bg_k^s$ is constructed in the proof of Proposition~\ref{prop:stoch_approx_exp_log}).
\paragraph{Assumption on $(\bg_k^s)$ and $(\bgamma_k)$.} Immediate by Proposition~\ref{prop:stoch_approx_exp_log}.

\paragraph{Assumption on $\bar{D}_s$.} As established before Proposition~\ref{prop:stoch_approx_exp_log}, $-\sm$ is a Lyapunov function to the DI associated with $D_s$ and the set $\cZ_s$. We now prove that $\sm(\cZ_s)$ is of zero-measure. It is a Sard's type result and is a simple adaptation of {\cite[Theorem 5.]{bolte2021conservative}}.

  We first apply Proposition~\ref{prop:var_strat_cons} to $\sm: \bbR^d \rightarrow \bbR$, $\bar{D}$ and $\cM = \bbS^{d-1}$, with $p \geq d$, obtaining $(\cX_i)$ a partition of $\bbR^d$ into $C^p$ manifolds that is compatible with $\bbS^{d-1}$. Then, noting that for any $u \in \cX_i \subset \bbS^{d-1}$, it holds that $\cT_{\cX_i}(u) \subset \cT_{\bbS^{d-1}}(u) = \{v - \scalarp{v}{u} : v \in \bbR^d \}$, we obtain
  \begin{equation*}
   \{ \nabla_{\cX_i} \sm(u)\} = P_{\cT_{\cX_i}(u)}(\bar{D}(u)) =P_{\cT_{\cX_i}(u)}\circ P_{\cT_{\bbS^{d-1}}(u)}(\bar{D}(u)) = P_{\cT_{\cX_i}(u)}(\bar{D}_s(u))\, ,
  \end{equation*}
   which implies
   \begin{equation*}
     \cZ_s \subset \bigcup_{i \in \cI} \{ u \in \cX_i: \nabla_{\cX_i} \sm(u) = 0\}\, .
   \end{equation*}
   Hence,
   \begin{equation*}
     \sm(\cZ_z) \subset \bigcup_{i \in \cI} \sm_{|\cX_i}\left(\{ w \in \cX_i: \nabla_{\cX_i} \sm(w) = 0\}\right)\, .
   \end{equation*}
   By Sard's theorem (\cite{sard1942measure}) every set in the union has zero-measure. Therefore, $\sm(\cZ_s)$ has zero-measure and the assumption on $\bar{D}_s$ holds.

  \paragraph{Assumption on $e_k = \eta_{k+1} + \gamma_k r_k$.} First, note that for any $k\leq l \leq \sn(k,T)$,
  \begin{equation*}
    \norm{\sum_{i=k}^l \bgamma_i^2 r_i} \leq \left(\sup_{k \leq j \leq l} \norm{\bgamma_j r_j}\right)\sum_{i=k}^l \bgamma_i \leq \left(\sup_{k \leq j \leq l} \norm{\bgamma_j r_j}\right)\sum_{i=k}^{\sn(k,T)} \bgamma_i \leq T \sup_{k \leq j \leq l} \norm{\bgamma_j r_j}\, ,
  \end{equation*}
and the right-hand side goes to zero \edi{on $\cE$}, when $k \rightarrow + \infty$. Thus, we only need to prove that the following quantity goes to zero \edi{on $\cE$}
\begin{equation*}
 \sC(k):= \sup\left\{ \norm{\sum_{i=k}^l \bgamma_i \bar{\eta}_{i+1}} : k \leq l \leq \sn(k,T)\right\} \, .
\end{equation*}

To prove that $\sC(k)$ goes to zero we want to apply \cite[Proposition 4.2, Remarks 4.3]{benaim2006dynamics}, which states that it is the case if there is a filtration $(\cF_k)$ such that \emph{i)} $(\bgamma_k)$ and $(\bar{\eta}_{k})$ are adapted to $\cF_k$, \emph{ii)} there is $q \geq 2$ such that $\sup_{k} \bbE[\norm{\bar{\eta}_{k+1}}^q] < + \infty$, \emph{iii)} and $\sum_{k} \bbE[\bgamma_k^{1+q/2}] <  +\infty$.

Note that the first two points hold by Proposition~\ref{prop:stoch_approx_exp_log} (for any $q$) but for the last point, even though we have that, \edi{on $\cE$}, almost surely $\sum_{k}\bgamma_k^{1 + 2/c_3} \leq \sum_{k} k^{-2} < + \infty$, \edi{such bound holds only on $\cE$, and, furthermore,} $c_3$ is a random variable, \edi{therefore we do not necessarily have such bound in expectation for a fixed $q$.}

To address this issue, for every $b \in \bbQ$ define $\hat{\gamma}_k(b) = \min(\bgamma_k, k^{-b})$ and 
\begin{equation*}
  \tilde{\sC}(k,b):= \sup\left\{ \norm{\sum_{i=k}^l \hat{\gamma}_i(b)\eta_{i+1}} : k \leq l \leq \sn(k,T)\right\} \, .
 \end{equation*}
Next, for any $k_0 \in \bbN$, define $A_{k_0, b}$ as the following event 
\begin{equation*}
  A_{k_0, b} = [\forall k \geq k_0\, , \bgamma_k = \hat{\gamma}_k(b)]\, .
\end{equation*}
Note that by construction (since $b$ is fixed) assumptions of \cite[Proposition 4.2, Remarks 4.3]{benaim2006dynamics} are satisfied for $\bar{\eta}_{k+1},\hat{\gamma}_k(b)$. Therefore, $ \tilde{\sC}(k,b) \rightarrow 0$. Moreover, on $A_{k_0, b}$, for $k \geq k_0$, $\tilde{\sC}(k,b) = \sC(k)$. Therefore, on $A_{k_0, b}$, $\sC(k)$ converges to zero.

Finally, \edi{from the properties of $(\bgamma_k)$ on $\cE$}, $\edi{\cE \subset}\bigcup_{k_{0}, b\in \bbN}A_{k_{0},b}$. Hence, $\sC(k,b)$ converges to zero \edi{on $\cE$}.

This completes the proof of the assumption on $e_k = \eta_{k+1} + \gamma_k r_k$ and thus the proof of Theorem~\ref{thm:main}. 

\section{Proof of Proposition~\ref{prop:log_wk}}\label{sec:pf_logwk}

Note that for the exponential loss $l(q)=  e^{-q}$ and the logistic loss $l(q) = \log(1+ e^{-q})$, it holds 
\begin{align}\label{lm:exp_log}
\frac{e^{-q}}{2}\leq -l'(q) \leq e^{-q}, \quad \text{ for all } q \geq 0.
\end{align}
In the following proofs, we assume $k$ is sufficiently large so that $\sm(w_{k}) \geq \varepsilon \norm{w_k}^L$, for some $\varepsilon >0$. Taking for instance $\varepsilon = \liminf \sm(u_k)/2$, this is always possible by \edi{the fact that we are working on the event $\cE$}. We now establish that $(\norm{w_k})_k$ strictly increases to infinity.
\begin{lemma}\label{lm:wk_infty}
  Under Assumptions~\ref{hyp:loss_exp_log}--\ref{hyp:conserv}, \edi{on the event $\cE$}, it holds that $(\norm{w_k})$ is a strictly increasing sequence that diverges to infinity.
\end{lemma}
\begin{proof}
  Using the fact that $l'(q) <0$, we obtain
  \begin{equation*}
  \begin{split}
    \norm{w_{k+1}}^2 &= \norm{w_k}^2 - \frac{2\gamma}{n_b} \sum_{i\in B_k} l'(p_i(w_k)) \scalarp{\sa_i(w_k)}{w_k} + \frac{\gamma^2}{n_b^2} \norm{\sum_{i\in B_k}l'(p_i(w_k)) \sa_i(w_k)}^2 \\
    &\geq \norm{w_k}^2 - L\frac{2\gamma}{n_b} \sum_{i\in B_k} l'(p_i(w_k)) p_i(w_k) \\
    &\geq \norm{w_k}^2 - L \frac{2\gamma}{n_b} \sum_{i\in B_k} l'(p_i(w_k)) m(w_k) \\
    &\geq \norm{w_k}^2 - \varepsilon L \frac{2\gamma}{n_b} \sum_{i \in B_k} l'(p_i(w_k))\norm{w_k}^L\, ,
  \end{split}
\end{equation*}
where for the second inequality we used Assumption~\ref{hyp:conserv} and for the last the homogeneity of the margin.
As a result, $(\norm{w_k})_k$ is a strictly increasing sequence. Assume by contradiction that $\sup_{k \geq k_0} \norm{w_k} \leq M$. Then, there exists $\delta > 0$ such that $\inf_{k \geq k_0\, , i \leq n} (-l'(p_i(w_k))) \geq \delta >0$. In particular, it implies that for all $k \geq k_0$,
\begin{equation*}
  \norm{w_{k+1}}^2 \geq \norm{w_k} + 2\varepsilon L \gamma \delta \norm{w_{k_0}}^L\, ,
\end{equation*}
which contradicts the fact that $\sup_{k \geq k_0} \norm{w_k} \leq M$. Therefore, $\norm{w_k} \rightarrow + \infty$.
\end{proof}
We now turn to the proof of the first part of Proposition~\ref{prop:log_wk}.
\begin{proposition}
  Let Assumptions~\ref{hyp:loss_exp_log}--\ref{hyp:conserv} hold. Almost surely \edi{on the event $\cE$}, there exist constants $c_1,c_2 >0$ and $k_2 \in \bbN$ such that for all $k \geq k_2$, 
  \begin{equation*}
   c_1 \log(k) \leq  \norm{w_k}^L \leq c_2 \log(k)\, .
  \end{equation*}
\end{proposition}
In the following proof, $ C, C', C_1, C_2, \dots $ denote positive constants that may vary from line to line. Additionally, we repeatedly use the following crude estimate, which follows from Taylor expansion: for $a > 0$ and sufficiently small $ x \in \mathbb{R} $, there exist constants $C, C' > 0$ such that  
\begin{equation*}
  1 + C' x \leq (1 + x)^a \leq 1 + Cx \, .
\end{equation*}  
\begin{proof}
  By Assumption~\ref{hyp:conserv}, for every $k$, $\scalarp{\sa_i(w_k)}{w_k} = Lp_i(w_k) =L \norm{w_k}^L p_i(u_k)$. Therefore,
  \begin{equation}\label{eq:est_wk}
    \begin{split}
      \norm{w_{k+1}}^2 &= \norm{w_k}^2 - 2\frac{\gamma}{n_b} \sum_{i\in B_k} l'(p_i(w_k)) \scalarp{\sa_i(w_k)}{w_k} + \frac{\gamma^2}{n_b^2} \norm{\sum_{i\in B_k}l'(p_i(w_k)) \sa_i(w_k)}^2\\
      &=\norm{w_k}^2 - 2L \norm{w_k}^L \frac{\gamma}{n_b} \sum_{i \in B_k} l'(p_i(w_k)) p_i(u_k) + \frac{\gamma^2}{n_b^2} \norm{\sum_{i\in B_k}l'(p_i(w_k)) \sa_i(w_k)}^2\, .
    \end{split}
  \end{equation}
  Since each $D_i$ is locally bounded, it is bounded on $\bbR^{d-1}$, which implies that there is a constant $C >0$ such that for every $1 \leq i \leq n$, $u \in \bbS^{d-1}$ and $v \in D_i(u)$,
  \begin{equation*}
    \norm{v} + |p_i(u)| \leq C\, .
  \end{equation*} 
  In particular, by Assumption~\ref{hyp:conserv}, $\sa_i(w_k) \leq C \norm{w_k}^{L-1}$.
  Therefore, using Equation~\eqref{lm:exp_log} and the fact that $p_i(u_k) \geq \sm(u_k) \geq \varepsilon$,
  \begin{equation*}
    \begin{split}
      \norm{w_{k+1}}^2 \leq \norm{w_k}^2 +  2CL\norm{w_k}^L \frac{\gamma}{n_b}\sum_{i \in B_k} e^{-\varepsilon \norm{w_k}^L}p_i(u_k)  + C''\frac{\gamma^2\norm{w_k}^{2(L-1)}}{n_b^2} \sum_{i\in B_k} e^{-2\varepsilon \norm{w_k}^L} \,.
    \end{split}
  \end{equation*}
Since $\norm{w_k} \rightarrow + \infty$ and $p_i(u_k) \geq \varepsilon$, there exists a constant $C>0$
  \begin{equation*}
    \norm{w_{k+1}}^2 \leq \norm{w_k}^2 + C \gamma\norm{w_k}^L e^{-\varepsilon \norm{w_k}^L}\, .
  \end{equation*}
  Thus, using the Taylor's expansion of $(1+x)^{L/2}$ near zero, for $k$ large enough,
  \begin{equation*}
    \begin{split}
    \norm{w_{k+1}}^L &\leq \norm{w_k}^L \left( 1 + C\norm{w_k}^{L-2} e^{-\varepsilon\norm{w_k}^L } \right)^{L/2}\\
    &\leq \norm{w_k}^L \left(1 + C' \norm{w_k}^{L-2}e^{-\varepsilon\norm{w_k}^L } \right)\\
    &\leq \norm{w_k}^L + C'' e^{-\varepsilon\norm{w_k}^L /2}\, . 
  \end{split}
  \end{equation*}
 This implies
  \begin{equation*}
    e^{\frac{\varepsilon}{2} \norm{w_k}^L}\left(\norm{w_{k+1}}^L  - \norm{w_k}^L\right)\leq C \, .
  \end{equation*}
  Denote $m_k:=\norm{w_{k+1}}^L$. Since $0 \leq m_{k+1} - m_k \rightarrow 0$, $ m_k \leq t \leq m_{k+1}$ implies $t \leq 2m_k$ for $k$ large enough. Therefore,
  \begin{equation*}
    \int_{m_k}^{m_{k+1}} e^{\varepsilon t/4}\dif t \leq   e^{\varepsilon \norm{w_k}^L/2}\left(\norm{w_{k+1}}^L  - \norm{w_k}^L\right)  \leq C \, .
  \end{equation*}
  In particular, for any $k_0$, such that for all $k \geq k_0$, all the preceding equations hold, we obtain,
  \begin{equation*}
    \frac{4(e^{m_k} -e^{m_{k_0}}) }{\varepsilon} = \int_{m_{k_0}}^{m_{k}} e^{\varepsilon t/4}\dif t  \leq C(k-k_0)\, .
  \end{equation*}
  Therefore, for $k$ large enough,
  \begin{equation*}
    m_{k}\leq C_1 + \log(C_2(k-k_0)) \leq C_3 \log(k)\, ,
  \end{equation*}
  which proves the upper bound.

Establishing the lower bound is similar. First, using Equation~\eqref{eq:est_wk}, we obtain
  \begin{equation*}
    \begin{split}
      \norm{w_{k+1}}^2 &\geq \norm{w_k}^2 + \frac{L \varepsilon\gamma \norm{w_k}^L}{n_b}\sum_{i\in B_k}e^{-\norm{w_k}^Lp_i(u_k)}\\
  &\geq \norm{w_k}^2 + L \varepsilon \gamma  \norm{w_k}^L e^{-C\norm{w_k}^L}\, ,
    \end{split}
  \end{equation*}
  where $C>0$ is such that for every $i \leq n$ and $u \in \bbS^{d-1}$, $|p_i(u)| \leq C$.

Therefore, for $k$ large enough, 
\begin{equation*}
  \begin{split}
    \norm{w_{k+1}}^L &\geq \norm{w_k}^L(1 + C_1 \norm{w_k}^{L-2} e^{-C\norm{w_k}^L})^{L/2} \\
    &\geq \norm{w_k}^L\left( 1 + C_2 \norm{w_k}^{L-2}e^{-C\norm{w_k}^L}\right)\\
    &\geq \norm{w_k}^L + C_3 e^{-C \norm{w_k}^L}\, .
  \end{split}
\end{equation*}
Thus,
\begin{equation*}
  e^{ C\norm{w_k}^L} \left(\norm{w_{k+1}}^L - \norm{w_k}^L\right)\geq C_3\, .
\end{equation*}
Recalling that $m_k = \norm{w_k}^L$, similarly to previous computations, for $m_{k}\leq t\leq m_{k+1} $, we obtain 
\begin{equation*}
  \int_{m_k}^{m_{k+1}}e^{Ct} \dif t\geq C_3\, .
\end{equation*}
Therefore, for a fixed, large enough $k_0$ and $k \geq k_0$,  
\begin{equation*}
\frac{e^{C m_{k}}}{ C}\geq \int_{m_{k_0}}^{m_{k}} e^{C t} \dif t \geq C_3(k - k_0)\, ,
\end{equation*}
which implies, for $C_2$ small enough and $k$ large enough,
\begin{equation*}
  m_{k} \geq \frac{\log(C_3  C(k-k_0))}{C} \geq C_2 \log(k)\, .
\end{equation*}
\end{proof}

To finish the proof of Proposition~\ref{prop:log_wk} it is enough to notice that \edi{on $\cE$,} for $k$ large enough,
\begin{equation*}
  \cL(w_k) = \frac{1}{n}\sum_{i=1}^n e^{-\norm{w_k}^L p_i(u_k)} \leq e^{-\sm(u_k) \norm{w_k}^L} \leq e^{-\varepsilon \norm{w_k}^L } \leq \frac{1}{k^{\varepsilon c_1}}\, .
\end{equation*}

\section{General convergence setting}\label{app:gen_sett}
The purpose of this section is to present a general convergence setting under which it is possible to apply the stochastic approximation ideas that were used in the proof of Proposition~\ref{prop:stoch_approx_exp_log} and Theorem~\ref{thm:main}. 

Mainly, our setting applies to more general losses than the exponential or logistic and allows to treat the case where we do not have an a priori control of the form $\liminf \sm(u_k) >0$. Due to the generality of the approach, we do not aim to push its limits and believe that better guarantees (such as an equivalent version of Proposition~\ref{prop:log_wk}) could be obtained on a case-by-case basis.

We still analyze~\eqref{eq:sgd_new} but now allow step-sizes to decrease to zero:
\begin{equation}\label{eq:sgd_gen}
  w_{k+1} = w_k -  \frac{\gamma_k}{n_b}\sum_{i\in B_k} l'(p_i(w_k)) \sa_i(w_k)\, .
\end{equation}
Note that~\eqref{eq:sgd_gen} encompasses the deterministic setting, in which $n_b = 0$, and note the presence of $(\gamma_k)$.

We first state our assumptions that are, mostly, only mild modifications of Assumptions~\ref{hyp:loss_exp_log}--\ref{hyp:conserv}.
\begin{assumption}\label{Hgen:loss}\phantom{=}
  \begin{enumerate}[label=\roman*)]
    \item There exists a positive integer $L \in \mathbb{N}^*$,  such that, for every $1 \leq i \leq n$, the function $p_i$ is $L$-homogeneous, locally Lipschitz continuous and semialgebraic.
    \item\label{Hgen:loss_gen} The loss function $l$ is $C^1$, with $l'(q) <0$ on some interval $[q_0, + \infty)$ and is such that for any two sequences $a_k,b_k \rightarrow + \infty$,
    \begin{equation*}
        \limsup \frac{b_k}{a_k}< 1 \implies \frac{l'(a_k)}{l'(b_k)} \rightarrow 0 \, .
    \end{equation*}
  \end{enumerate}
\end{assumption}

\begin{assumption}\phantom{=}\label{Hgen:cons}
  For every $i \leq n$, $\sa_i$ is measurable and $D_i$ is semialgebraic. Moreover, for every $w \in \bbR^d$ and $\lambda \geq 0$, $\sa_i(w)  \in D_i(w)$ and
  \begin{equation*}
    D_i(\lambda w) = \lambda^{L-1} D_i(w)\, .
  \end{equation*}
\end{assumption}

\begin{assumption}[Deterministic setting]\label{Hgen:det}
  The batch-size $n_b = n$ and there is $\gamma>0$ such that $\gamma_k \leq \gamma$ and $\sum \gamma_k = + \infty$.
\end{assumption}
\begin{assumption}[Stochastic setting]\label{Hgen:sto}
 The sequence $(\gamma_k)$ is a sequence of strictly positive step-sizes such that $\sum_{k} \gamma_k = + \infty$ and there is $q \geq 2$ such that $\sum_{k} \gamma_k^{1+q/2} < + \infty$.
\end{assumption}
\edi{Finally, all of the results now hold on the, a priori, larger event
\begin{equation*}
  \cE' = [ \norm{w_k} \rightarrow + \infty] \cap [\norm{w_k}^{L-2} \sum_{i=1}^n \ell'(p_i(w_k)) \rightarrow 0] \cap [\textrm{there is $k_0 \in \bbN$ such that $\inf_{k \geq k_0}\sm(w_k)\geq q_0$}]\, ,
\end{equation*}
where $q_0$ was defined in in Assumption~\ref{Hgen:loss}.
}

\paragraph{Comments on Assumptions~\ref{Hgen:loss}--\ref{Hgen:cons} \edi{and the event $\cE'$}.} These assumptions are similar to Assumptions~\ref{hyp:loss_exp_log}--\ref{hyp:conserv}. The main differences are the following. 

\begin{itemize}
  \item Assumption~\ref{Hgen:loss} allows us to consider more general losses such as $l(q) = e^{-q^a}$ or $l(q) =\log (1+e^{-q^a})$, for $a>0$, already considered in \cite{Lyu_Li_maxmargin}.
  \item Differently to the main setting of the paper, to treat the stochastic case we require in Assumption~\ref{Hgen:sto} vanishing stepsizes. In fact it could be alleviated, if, similarly to Proposition~\ref{prop:log_wk} we have control on the growth of $\norm{w_k}$. In the deterministic setting of Assumption~\ref{Hgen:det} one could still choose a constant stepsize.
  \item In \edi{the definition of the event $\cE'$}, differently to the main setting of the paper, we \emph{assume} $\norm{w_k} \rightarrow + \infty$. Note that, similarly to Lemma~\ref{lm:wk_infty}, we could ensure this growth if we assumed $\liminf \sm(u_k) >q_0$. However, the \edi{event $\cE'$} allows us to treat the case, where, a priori, we do not have such control on the growth of the normalized margin. For instance, if $l(q) = e^{-q}$, then \edi{the event $\cE'$} requires that $\norm{w_k}^{L-2}e^{-\norm{w_k}^L \sm(u_k)} \rightarrow 0$. Thus, it would hold as soon as $\sm(u_k) \geq \norm{w_k}^{-L}\log(\norm{w_k}^{L-1})$ a much softer requirement than $\liminf \sm(u_k) >0$. 
\end{itemize}

Under these new assumptions we have a version of Proposition~\ref{prop:stoch_approx_exp_log}. Recall that $\cF_k$ denotes the sigma algebra generated by $\{w_0, \ldots, w_k \}$.

\begin{proposition}\label{Pgen:stoch_approx}
  Let either Assumptions~\ref{Hgen:loss}--\ref{Hgen:det} or Assumptions~\ref{Hgen:loss},\ref{Hgen:cons} and \ref{Hgen:sto} hold. There are sequences $(\bg_k), (r_k), (\bgamma_k), (\bar{\eta}_{k+1})$ such that
  \begin{equation}\label{Egen:stoch_app_u}
    u_{k+1} = u_k + \bgamma_k(\bg_k - \scalarp{\bg_k}{u_k}u_k) + \bgamma_k \bar{\eta}_{k+1} + \bgamma_k^2 r_k\, ,
  \end{equation}
  and the following holds.
  \begin{enumerate}
    \item\label{Gpr_res:rk} \edi{On the event $\cE'$,} the sequence $(r_k)$ is such that almost surely $\sup_{k}\norm{r_k} < + \infty$.
    \item\label{Gpr_res:gammak} The sequence $(\bgamma_k)$ is adapted to $(\cF_k)$ and almost surely \edi{on $\cE'$}, there is $k_0 \in \bbN$, such that for all $k \geq k_0$, $\bgamma_k >0$. Moreover, \edi{on $\cE'$}, $\bgamma_k \rightarrow 0$ and $\sum_{k} \bar{\gamma}_k = + \infty$ and under Assumption~\ref{Hgen:sto}, we additionally have that $\sup_k \bgamma_k / \gamma_k <  +\infty$.
    \item\label{Gpr_res:etak} Under Assumption~\ref{Hgen:det}, $\bar{\eta}_{k} \equiv 0$. Otherwise, under Assumption~\ref{Hgen:sto}, the sequence $(\bar{\eta}_{k})$ is adapted to $(\cF_k)$, 
    \begin{equation*}
    \bbE[\bar{\eta}_{k+1} |\cF_k] = 0 \, ,
    \end{equation*}
    and there is a deterministic constant $c_1>0$ such that $\sup_{k} \norm{\bar{\eta}_{k+1}} < c_6$.
    \item\label{Gpr_res:barD} \edi{On $\cE'$,} for any unbounded sequence $(k_j)_{j \geq 0}$, such that $u_{k_j}$ converges to $u \in \bbS^{d-1}$, it holds that $\dist(\bar{D}(u), \bg_{k_j}) \rightarrow 0$. 
  \end{enumerate}
\end{proposition}

\begin{proof}
  The proof goes as the one of Proposition~\ref{prop:stoch_approx_exp_log} in Section~\ref{pf:sto_app_explog} but with $\gamma$ replaced by $\gamma_k$. The claim on $(\bgamma_k)$ follows from the fact that 
  \begin{equation*}
    \bgamma_k = -\gamma_k \norm{w_k}^{L-2} \sum_{j=1}^n l'(p_j(w_k))\, .
  \end{equation*}
  Indeed, \edi{on $\cE'$}, for $k$ large enough, $p_j(w_k) \geq q_0$. Hence, $\bgamma_k >0$ by Assumption~\ref{Hgen:loss}-\ref{Hgen:loss_gen}. Moreover, it goes to zero \edi{on $\cE'$}. From the same assumption and under Assumption~\ref{Hgen:sto} we indeed have $\limsup \bgamma_k/\gamma_k < + \infty$. 
  The claim on $\bar{D}_s$ is obtained from Assumption~\ref{Hgen:loss}-\ref{Hgen:loss_gen}. The other claims are obtained \emph{exactly as} in the proof of  Proposition~\ref{prop:stoch_approx_exp_log}.
\end{proof}

As a result, we have a version of Theorem~\ref{thm:main}.
\begin{theorem}\label{Tgen:main}
  In the setting of Proposition~\ref{Pgen:stoch_approx}, almost surely \edi{on $\cE'$}, $\sm(u_k)$ converges to a nonnegative limit and 
  \begin{equation}\label{Egen:conv_uk}
    \dist(u_k, \cZ_s) \xrightarrow[k \rightarrow + \infty]{} 0 \, .
  \end{equation}
\end{theorem}
\begin{proof}
  The proof goes as the proof of Theorem~\ref{thm:main}, with Proposition~\ref{Pgen:stoch_approx} playing the role of Proposition~\ref{prop:stoch_approx_exp_log}.
\end{proof}

Let us finish this section with a brief comment on KKT points of~\eqref{def:prob2} considered in \cite{Lyu_Li_maxmargin}.
A point $w \in \bbR^d$ is a KKT point of~\eqref{def:prob2} if there is $\alpha_i, \ldots, \alpha_n \geq 0$ and $v_1, \ldots, v_n$, with $v_i \in \partial p_i(w)$ and 
\begin{equation*}
  w = \sum_{i=1}^n \alpha_i v_i \quad \textrm{ and for all $i$, } \quad \lambda_i(p_i(w)- 1) = 0 \, .
\end{equation*}
Denoting $\lambda_i = \alpha_i/ \sum_{j=1}^n \alpha_j$ and $u = w/ \norm{w}$, we can successively observe that \emph{i)} $\lambda_i \neq 0$ only if $p_i(u) = \sm(u)$ and $v_i = \norm{w}^L g_i$, with $g_i \in\partial p_i(u)$, \emph{ii)} $\sum_{i=1}^n \lambda_i g_i \in \bar{D}(u)$ \emph{iii)} therefore, $u= \sum_{i=1}^n \lambda_i g_i$ and finally $0\in \bar{D}_s(u)$. 
Conversely, if $u \in \cZ_s$ and \emph{if additionally} $\sm(u) >0$, then $w = u/\sm(u)^{1/L}$ is a KKT point of~\eqref{def:prob2}.

Nevertheless, due to the slow decrease rate required by the event $\cE'$, we do not, a priori, have that $\lim \sm(u_k) > 0$. Therefore, in the setting of Theorem~\ref{Tgen:main} there could be two scenarios. First, in which $\lim \sm(u_k) >0$ and in that case the limit points of $(u_k)$ are \emph{exactly} (scaled) KKT directions of~\eqref{def:prob2}. Second, in which $\lim \sm(u_k) =0$, and we only have that any limit point of $(u_k)$ is in $\cZ_s$.

Note however, that without any discussion on KKT points in all cases being in $\cZ_s$ is a meaningful description of optimality of maximization of the margin, which, furthermore, naturally extends to settings, where $D_i$ are different of the Clarke subgradients of $p_i$.

\end{document}